%% file: main.tex
\documentclass{svproc}
\usepackage{url}
\usepackage{amsmath}
\usepackage{amsfonts}
\usepackage{color}
\usepackage{array}
\usepackage{graphicx}
\graphicspath{ {./figs/} }
\usepackage{wrapfig}
\usepackage{mathtools}
\usepackage{pst-node}
\usepackage[ruled,linesnumbered]{algorithm2e}
\usepackage{multirow}
\usepackage[noadjust]{cite}
\usepackage{chngcntr}
\usepackage{apptools}
\AtAppendix{\counterwithin{theorem}{section}}
\AtAppendix{\counterwithin{corollary}{section}}
\AtAppendix{\counterwithin{definition}{section}}
\usepackage{enumitem}

\usepackage{amsthm}
\usepackage{caption}
\usepackage{picins}
\newcommand{\task}{\Pi}

\newcommand{\traj}{\xi}
\newcommand{\state}{x}
\newcommand{\statespace}{\mathcal{X}}
\newcommand{\safeset}{\mathcal{S}}
\newcommand{\unsafeset}{\mathcal{A}}
\newcommand{\unsafesure}{{\mathcal{A}_l}}
\newcommand{\unsafesurelearned}{{\mathcal{A}_l^\textrm{rec}}}
\newcommand{\umax}{\Delta\state}
\newcommand{\sd}{\textsf{sd}}
\newcommand{\occ}{\mathcal{O}}
\newcommand{\occp}{\mathcal{\tilde O}}

\newcommand{\frs}{\textsf{FRS}}

\newcommand{\grid}{z}
\newcommand{\gridrv}{Z}
\newcommand{\gridspace}{\mathcal{Z}}
\newcommand{\reals}{\mathbb{R}}
\newcommand{\expectation}{\mathbb{E}}
\newcommand{\numsafe}{N_s}
\newcommand{\numunsafe}{N_{\neg s}}
\newcommand{\control}{u}
\newcommand{\controlset}{\mathcal{U}}
\newcommand{\trajset}{\mathcal{T}}
\newcommand{\constraintset}{\mathcal{C}}
\newcommand{\unsafetrajset}{\mathcal{T}_\mathcal{A}}
\newcommand{\dynfeasibleset}{\mathcal{D}^{\xi_{xu}}}
\newcommand{\controlfeasibleset}{\mathcal{U}^{\xi_{xu}}}
\newcommand{\dynproject}{D}

\newcommand{\trajxu}{\traj_{xu}}
\newcommand{\trajx}{\traj_\state}
\newcommand{\traju}{\traj_\control}
\newcommand{\taskspace}{\mathcal{P}}
\newcommand{\constraintspace}{\mathcal{C}}

\theoremstyle{remark}
\newtheorem*{rem}{Remark}

\begin{document}
\mainmatter              
\title{Learning Constraints from Demonstrations}
\titlerunning{Learning Constraints from Demonstrations}  
%
\author{Glen Chou, Dmitry Berenson, Necmiye Ozay}
\authorrunning{Glen Chou, Dmitry Berenson, Necmiye Ozay} 
%
\tocauthor{Glen Chou, Dmitry Berenson, Necmiye Ozay}
\institute{Dept. of Electrical Engineering and Computer Science, \\University of Michigan, Ann Arbor, MI, 48109, USA\\ \email{\{gchou,dmitryb,necmiye\}@umich.edu}}

\maketitle              

\begin{abstract}
We extend the learning from demonstration paradigm by providing a method for learning unknown constraints shared across tasks, using demonstrations of the tasks, their cost functions, and knowledge of the system dynamics and control constraints. Given safe demonstrations, our method uses hit-and-run sampling to obtain lower cost, and thus unsafe, trajectories. Both safe and unsafe trajectories are used to obtain a consistent representation of the unsafe set via solving an integer program. Our method generalizes across system dynamics and learns a guaranteed subset of the constraint. We also provide theoretical analysis on what subset of the constraint can be learnable from safe demonstrations. We demonstrate our method on linear and nonlinear system dynamics, show that it can be modified to work with suboptimal demonstrations, and that it can also be used to learn constraints in a feature space.
\end{abstract}
%

\vspace{-30pt}
\section{Introduction}
\input{intro}
\vspace{-15pt}
\section{Related Work}
\input{related_work}
\vspace{-12pt}
\section{Preliminaries and Problem Statement}
\input{preliminaries}

\input{method}

\input{method_trunc}
\input{evaluations}
\vspace{-20pt}
\section{Conclusion}
\vspace{-5pt}
In this paper we propose an algorithm that learns constraints from demonstrations, which acts as a complementary method to IOC/IRL algorithms. We analyze the properties of our algorithm as well as the theoretical limits of what subset of an unsafe set can be learned from safe demonstrations. The method works well on a variety of system dynamics and can be adapted to work with suboptimal demonstrations. We further show that our method can also learn constraints in a feature space. The largest shortcoming of our method is the constraint space gridding, which yields a complex constraint representation and causes the method to scale poorly to higher dimensional constraints. We aim to remedy this issue in future work by developing a grid-free counterpart of our method for convex unsafe sets, which can directly describe standard pose constraints like task space regions \cite{tsr}.
\vspace{-25pt}

\section*{Acknowledgements}
\vspace{-5pt}
This work was supported in part by a Rackham first-year graduate fellowship, ONR grants N00014-18-1-2501 and N00014-17-1-2050, and NSF grants CNS-1446298, ECCS-1553873, and IIS-1750489.
\vspace{-15pt}
\bibliography{refs}
\bibliographystyle{abbrv}

\appendix
\include{analysis}

\end{document}

%% file: intro
\vspace{-5pt}
Inverse optimal control and inverse reinforcement learning (IOC/IRL) \cite{irl_1, irl_2, lfd3, ng_irl} have proven to be powerful tools in enabling robots to perform complex goal-directed tasks. These methods learn a cost function that replicates the behavior of an expert demonstrator when optimized. However, planning for many robotics and automation tasks also requires knowing constraints, which define what states or trajectories are safe. For example, the task of safely and efficiently navigating an autonomous vehicle can naturally be described by a cost function trading off user comfort and efficiency and by the constraints of collision avoidance and executing only legal driving behaviors. In some situations, constraints can provide a more interpretable representation of a behavior than cost functions. For example, in safety critical environments, recovering a hard constraint or an explicit representation of an unsafe set in the environment is more useful than learning a ``softened" cost function representation of the constraint as a penalty term in the Lagrangian. Consider the autonomous vehicle, which absolutely must avoid collision, not simply give collisions a cost penalty. Furthermore, learning global constraints shared across many tasks can be useful for generalization. Again consider the autonomous vehicle, which must avoid the scene of a car accident: a shared constraint that holds regardless of the task it is trying to complete.

While constraints are important, it can be impractical for a user to exhaustively program into a robot all the possible constraints it should obey when performing its repertoire of tasks. To avoid this, we consider in this paper the problem of recovering the latent constraints within expert demonstrations that are shared across tasks in the environment. Our method is based on the key insight that each safe, optimal demonstration induces a set of lower-cost trajectories that must be unsafe due to violation of an unknown constraint. Our method samples these unsafe trajectories, ensuring they are also consistent with the known constraints (system dynamics, control constraints, and start/goal constraints), and uses these unsafe trajectories together with the safe demonstrations as constraints in an ``inverse" integer program which recovers a consistent unsafe set. Our contributions are fourfold:
\vspace{-4pt}
\begin{itemize}
	\item We pose the novel problem of learning a shared constraint across tasks.
	\item We propose an algorithm that, given known constraints and boundedly suboptimal demonstrations of state-control sequences, extracts unknown constraints defined in a wide range of constraint spaces (not limited to the trajectory or state spaces) shared across demonstrations of different tasks.
	\item We provide theoretical analysis on the limits of what subsets of a constraint can be learned, depending on the demonstrations, the system dynamics, and the trajectory discretization. We also show that our method can recover a guaranteed underapproximation of the constraint.
	\item We provide experiments that justify our theory and show that our algorithm can recover an unsafe set with few demonstrations, across different types of linear and nonlinear dynamics, and can be adapted to work with boundedly suboptimal demonstrations. We also demonstrate that our method can learn constraints in the state space and a feature space. 
\end{itemize}
\vspace{-5pt}

%% file: related_work.tex
\vspace{-5pt}
\textbf{Inverse optimal control} \cite{kalman, boyd} (IOC) and \textbf{inverse reinforcement learning} (IRL) \cite{ng_irl} aim to recover an objective function consistent with the received expert demonstrations, in the sense that the demonstrations (approximately) optimize the cost function. Our method is complementary to these approaches; if the demonstration is solving a constrained optimization problem, we are finding its constraints, given the objective function; IOC/IRL finds the objective function, given its constraints. For example, \cite{toussaint} attempts to learn the cost function of a constrained optimization problem from optimal demonstrations by minimizing the residuals of the KKT conditions, but the constraints themselves are assumed known. Another approach \cite{satinder} can represent a state-space constraint shared across tasks as a penalty term in the reward function of an MDP. However, when viewing a constraint as a penalty, it becomes unclear if a demonstrated motion was performed to avoid a penalty or to improve the cost of the trajectory in terms of the true cost function (or both). Thus, learning a constraint which generalizes between tasks with different cost functions becomes difficult. To avoid this issue, we assume a known cost function to explicitly reason about the constraint.

One branch of \textbf{safe reinforcement learning} aims to perform exploration while minimizing visitation of unsafe states. Several methods for safe exploration in the state space \cite{safe_exploration, krause, claire} use a Gaussian process (GP) to explore safe regions in the state space. These approaches differ from ours in that they use exploration instead of demonstrations. Some drawbacks to these methods include that unsafe states can still be visited, Lipschitz continuity of the safety function is assumed, or the dynamics are unknown but the safe set is known. Furthermore, states themselves are required to be explicitly labeled as safe or unsafe, while we only require the labeling of whole trajectories. Our method is capable of learning a binary constraint defined in other spaces, using only state-control trajectories.
 
There exists prior work in learning \textbf{geometric constraints} in the workspace. In \cite{vijayakumar}, a method is proposed for learning Pfaffian constraints, recovering a linear constraint parametrization. In \cite{shah}, a method is proposed to learn geometric constraints which can be described by the classes of considered constraint templates. Our method generalizes these methods by being able to learn a nonlinear constraint defined in any constraint space (not limited to the state space).

Learning \textbf{local trajectory-based constraints} has also been explored in the literature. The method in \cite{dmitry} samples feasible poses around waypoints of a single demonstration; areas where few feasible poses can be sampled are assumed to be constrained. Similarly, \cite{anca} performs online constraint inference in the feature space from a single trajectory, and then learns a mapping to the task space. The methods in \cite{lfdc1,lfdc2,lfdc3,lfdc4} also learn constraints in a single task. These methods are inherently local since only one trajectory or task is provided, unlike our method, which aims to learn a global constraint shared across tasks.

Our work is also relevant to \textbf{human-robot interaction}. In \cite{knepper}, implicit communication of facts between agents is modeled as an interplay between demonstration and inference, where ``surprising" demonstrations trigger inference of the implied fact. Our method can be seen as an inference algorithm which infers an unknown constraint implied by a ``surprising" demonstration.

%% file: preliminaries.tex
\vspace{-5pt}
The goal of this work is to recover unknown constraints shared across a collection of optimization problems, given boundedly suboptimal solutions, the cost functions, and knowledge of the dynamics, control constraints, and start/goal constraints. We discuss the forward problem, which generates the demonstrations, and the inverse problem: the core of this work, which recovers the constraints.
\vspace{-12pt}
\subsection{Forward optimal control problem}
\vspace{-8pt}
Consider an agent described by a state in some state space $\state \in \statespace$. It can take control actions $\control \in \controlset$ to change its state. The agent performs tasks $\task$ drawn from a set of tasks $\taskspace$, where each task $\task$ can be written as a constrained optimization problem over state trajectories in state trajectory space $\trajx \in \trajset^\state$ and control trajectories $\traju \in \trajset^\control$ in control trajectory space:
\begin{problem}[Forward problem / ``task" $\task$]\label{prob:fwd_prob}
\vspace{-5pt}
\begin{equation}\label{eq:fwdprob}
	\begin{array}{>{\displaystyle}c >{\displaystyle}l >{\displaystyle}l}
		\underset{\trajx, \traju}{\text{minimize}} & \quad c_\task(\trajx, \traju) &\\
		\text{subject to} & \quad \phi(\trajx, \traju) \in \safeset \subseteq \constraintspace\\
		& \quad \bar\phi(\trajx, \traju) \in \bar\safeset \subseteq \bar\constraintspace\\
		& \quad \phi_\task(\trajx, \traju) \in \safeset_\task \subseteq \constraintspace_\task\\
	\end{array}
\end{equation}
\end{problem}

\noindent where $c_\task(\cdot): \trajset^\state \times \trajset^\control \rightarrow \mathbb{R}$ is a cost function for task $\task$ and $\phi(\cdot, \cdot) : \trajset^\state \times \trajset^\control \rightarrow \constraintspace$ is a known feature function mapping state-control trajectories to some constraint space $\constraintspace$. $\bar\phi(\cdot,\cdot): \trajset^\state \times \trajset^\control \rightarrow \bar\constraintspace$ and $\phi_\task(\cdot, \cdot): \trajset^\state \times \trajset^\control \rightarrow \constraintspace_\task$ are known and map to potentially different constraint spaces $\bar\constraintspace$ and  $\constraintspace_\task$, containing a known shared safe set $\bar \safeset$ and a known task-dependent safe set $\safeset_\task$, respectively. $\safeset$ is an unknown safe set, and the inverse problem aims to recover its complement, $\unsafeset \doteq \safeset^c$, the ``unsafe" set. In this paper, we focus on constraints separable in time: $\phi(\xi_x, \xi_u) \in \mathcal{A} \Leftrightarrow \exists t\in\{1,\ldots, T\}\; \phi(\xi_x(t), \xi_u(t)) \in \mathcal{A}$, where we overload $\phi$ so it applies to the instantaneous values of the state and the input. An analogous definition holds for the continuous time case. Our method easily learns non-separable trajectory constraints as well\footnote{Write Problem \ref{prob:feasibility_program} constraints as sums over partially separable/inseparable feature components instead of completely separable components.}.

A demonstration, $\trajxu \doteq (\trajx, \traju)$, is a state-control trajectory which is a boundedly suboptimal solution to Problem \ref{eq:fwdprob}, i.e. the demonstration satisfies all constraints and its cost is at most a factor of $\delta$ above the cost of the optimal solution $\trajxu^*$, i.e. $c(\trajx^*, \traju^*) \le c(\trajx, \traju) \le (1+\delta)c(\trajx^*, \traju^*)$. Furthermore, let $T$ be a finite time horizon which is allowed to vary. If $\trajxu$ is a discrete-time trajectory ($\trajx = \{\state_1, \ldots, \state_T\}$, $\traju = \{\control_1, \ldots, \control_T\}$), Problem \ref{prob:fwd_prob} is a finite-dimensional optimization problem, while Problem \ref{prob:fwd_prob} becomes a functional optimization problem if $\trajxu$ is a continuous-time trajectory ($\trajx : [0, T] \rightarrow \statespace$, $\traju : [0, T] \rightarrow \controlset$). We emphasize this setup does not restrict the unknown constraint to be defined on the trajectory space; it allows for constraints to be defined on any space described by the range of some known feature function $\phi$.

We assume the trajectories are generated by a dynamical system $\dot x = f(\state, \control, t)$ or $x_{t+1} = f(\state_t, \control_t, t)$ with control constraints $\control_t \in \controlset$, for all $t$, and that the dynamics, control constraints, and start/goal constraints are known. 
We further denote the set of state-control trajectories satisfying the unknown shared constraint, the known shared constraint, and the known task-dependent constraint as $\trajset_\safeset$, $\trajset_{\bar\safeset}$, and $\trajset_{\safeset_\task}$, respectively.  Lastly, we also denote the set of trajectories satisfying all known constraints but violating the unknown constraint as $\trajset_\unsafeset$.
\vspace{-5pt}
\subsection{Inverse constraint learning problem}
\setlength{\intextsep}{5pt}%
\setlength{\columnsep}{5pt}%
\begin{wrapfigure}{R}{0.4\textwidth}
    \centering
    \vspace{-32pt}
    \includegraphics[width=0.35\textwidth]{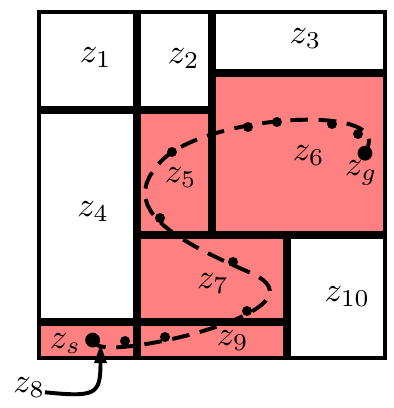}
    \vspace{-12pt}
    \caption{Discretized constraint space with cells $\grid_1, \ldots, \grid_{10}$. The trajectory's constraint values are assigned to the red cells.}
    \vspace{-3pt}
    \label{fig:discretization}
\end{wrapfigure}

The goal of the inverse constraint learning problem is to recover an unsafe set, $\unsafeset\subseteq\constraintset$, using $\numsafe$ provided safe demonstrations $\xi_{s_j}^*, j = 1, \ldots, \numsafe$, known constraints, and $\numunsafe$ inferred unsafe trajectories, $\xi_{{\neg s}_k}, k = 1, \ldots, \numunsafe$, generated by our method, which can come from multiple tasks. These trajectories can together be thought of as a set of constraints on the possible assigments of unsafe elements in $\constraintspace$. To recover a gridded approximation of the unsafe set $\unsafeset$ that is consistent with these trajectories, we first discretize $\constraintspace$ into a finite set of $G$ discrete cells $\gridspace \doteq \{\grid_1, \ldots, \grid_G\}$ and define an occupancy function, $\occ(\cdot)$, which maps each cell to its safeness: $\occ(\cdot) : \gridspace \rightarrow \{0, 1\}$, where $\occ(\grid_i) = 1$ if $\grid_i \in \unsafeset$, and $0$ otherwise. Continuous space trajectories are gridded by concatenating the set of grid cells $\grid_i$ that $\phi(\state_1), \ldots, \phi(\state_T)$ lie in, which is graphically shown in Figure \ref{fig:discretization} with a non-uniform grid. Then, the problem can be written down as an integer feasibility problem:

\begin{problem}[Inverse feasibility problem]\label{prob:feasibility_program}
\vspace{-6pt}\begin{equation}\label{eq:feasibility_program}
	\begin{array}{>{\displaystyle}c >{\displaystyle}l >{\displaystyle}l}
		\text{find} & \occ(\grid_1), \ldots, \occ(\grid_G) \in \{0, 1\}^G &\\
		\text{subject to} & \quad \sum_{\substack{\grid_i \in \{\phi(\traj^*_{s_j}(1)),\ldots,\\ \phi(\traj^*_{s_j}(T_j)) \} }} \occ(\grid_i) = 0, & \quad \forall j = 1, \ldots, \numsafe \\
		 & \quad \sum_{\substack{\grid_i \in \{\phi(\xi_{{\neg s}_k}(1)), \ldots,\\ \phi(\xi_{{\neg s}_k}(T_k))\}} } \occ(\grid_i) \ge 1, & \quad \forall k = 1, \ldots, \numunsafe
	\end{array}
\end{equation}
\end{problem}

Inferring unsafe trajectories, i.e. obtaining $\xi_{{\neg s}_k}, k=1, \ldots, \numunsafe$, is the most difficult part of this problem, since finding lower-cost trajectories consistent with known constraints that complete a task is essentially a planning problem. Much of the next section shows how to efficiently obtain $\xi_{{\neg s}_k}$. Further details on Problem \ref{prob:feasibility_program}, including conservativeness guarantees, incorporating a prior on the constraint, and a continuous relaxation can be found in Section \ref{sec:integer_program}.
\vspace{-5pt}
%
%

%% file: method.tex
\vspace{-2pt}
\section{Method}\label{sec:method}
\vspace{-2pt}
The key to our method lies in finding lower-cost trajectories that do not violate the known constraints, given a demonstration with boundedly-suboptimal cost satisfying all constraints. Such trajectories must then violate the unknown constraint. Our goal is to determine an unsafe set in the constraint space from these trajectories using Problem \ref{prob:feasibility_program}. In the following, Section \ref{sec:known_constraints} describes lower-cost trajectories consistent with the known constraints; Section \ref{sec:sampling} describes how to sample such trajectories; Section \ref{sec:version_space} describes how to get more information from unsafe trajectories; Section \ref{sec:integer_program} describes details and extensions to Problem \ref{eq:feasibility_program}; Section \ref{sec:suboptimality} discusses how to extend our method to suboptimal demonstrations. The complete flow of our method is described in Algorithm \ref{alg:total}.

\vspace{-8pt}
\subsection{Trajectories satisfying known constraints}\label{sec:known_constraints}

Consider the forward problem (Problem \ref{eq:fwdprob}). We define the set of unsafe state-control trajectories induced by an optimal, safe demonstration $\traj_{xu}^*$, $\unsafetrajset^{\trajxu^*}$, as the set of state-control trajectories of lower cost that obey the known constraints:
\begin{equation}
	\unsafetrajset^{\trajxu^*} \doteq \{ \trajxu \ |\ c(\trajx, \traju) < c(\trajx^*, \traju^*), \trajxu \in \trajset_{\bar\safeset}, \trajxu \in \trajset_{\safeset_\task} \}.
\end{equation}

In this paper, we deal with the known constraints from the system dynamics, the control limits, and task-dependent start and goal state constraints. Hence, $\trajset_{\bar\safeset} = \dynfeasibleset \cap \controlfeasibleset$, where $\dynfeasibleset$ denotes the set of dynamically feasible trajectories and $\controlfeasibleset$ denotes the set of trajectories using controls in $\controlset$ at each time-step. $\trajset_{\safeset_\task}$ denotes trajectories satisfying start and goal constraints. We develop the method for discrete time trajectories, but analogous definitions hold in continuous time. For discrete time, length $T$ trajectories, $\controlfeasibleset$, $\dynfeasibleset$, and $\trajset_{\safeset_\task}$ are explicitly:
\begin{equation}
\begin{aligned}
	\controlfeasibleset &\doteq \{ \xi_{xu} \ |\ \control_t \in \controlset, \ \forall t \in \{ 1, \ldots, T-1 \}\ \}, \\
	\dynfeasibleset &\doteq \{ \xi_{xu} \ |\ \state_{t+1} = f(\state_t, \control_t),\ \forall t \in \{ 1, \ldots, T-1 \}\ \}, \\
	\trajset_{\safeset_\task} &\doteq \{ \xi_{xu} \ |\ \state_{1} = \state_s, \state_T = \state_g \}.
\end{aligned}
\end{equation}
\vspace{-10pt}
\subsection{Sampling trajectories satisfying known constraints}\label{sec:sampling}
\vspace{-2pt}
\begin{table}
\begin{center}
\smaller
\begin{tabular}{ | c | c | c | c | }
\hline
\textbf{Dynamics} & \textbf{Cost function} & \textbf{Control constraints} & \textbf{Sampling method} \\ \hline
Linear & Quadratic & Convex & Ellipsoid hit-and-run (Section \ref{sec:elllipsoidsampling}) \\ \hline
Linear & Convex & Convex & Convex hit-and-run (Section \ref{sec:convexcost}) \\ \hline
\multicolumn{3}{| c |}{Else} & Non-convex hit-and-run (Section \ref{sec:nonconvexintersection}) \\ 
\hline
\end{tabular}
\end{center}
\vspace{-5pt}
\caption{Sampling methods for different dynamics/costs/feasible controls.}
\label{tbl:sampling}
\vspace{-25pt}
\end{table}
\vspace{0pt}
We sample from $\trajset_{\unsafeset}^{\trajxu^*}$ to obtain lower-cost trajectories obeying the known constraints using hit-and-run sampling \cite{hit_and_run} over the set $\trajset_{\unsafeset}^{\trajxu^*}$, a method guaranteeing convergence to a uniform distribution of samples over $\trajset_{\unsafeset}^{\trajxu^*}$ in the limit; the method is detailed in Algorithm \ref{alg:hnr} and an illustration is shown in Figure \ref{fig:hnr}. Hit-and-run starts from an initial point within the set, chooses a direction uniformly at random, moves a random amount in that direction such that the new point remains within the set, and repeats. 

Depending on the convexity of the cost function and the control constraints and on the form of the dynamics, different sampling techniques can be used, organized in Table \ref{tbl:sampling}. The following sections describe each sampling method.

\vspace{5pt}
\noindent\begin{minipage}{.59\textwidth}
\begin{algorithm}[H]\label{alg:hnr}
\SetAlgoLined
\SetKwInOut{Input}{Input}
\SetKwInOut{Output}{Output}
\Output{$\textsf{out} \doteq \{\traj_1, \ldots, \traj_{\numunsafe}\}$}
\Input{$\unsafetrajset^{\trajxu^*}, \trajxu^*, \numunsafe$}
 $\trajxu \leftarrow \trajxu^*$;  $\textsf{out} \leftarrow \{ \}$\;
 \For{i = 1:$\numunsafe$}{
  $r \leftarrow \textsf{sampleRandDirection}$\;
  $\mathcal{L} \leftarrow \unsafetrajset^{\trajxu^*} \cap \{\trajxu' \in \trajset \ |\ \trajxu' = \trajxu + \beta r\}$\;
  $L_-, L_+ \leftarrow \textsf{endpoints}(\mathcal{L})$\;
  $\trajxu \sim \textsf{Uniform}(L_-, L_+)$\;
  $\textsf{out} \leftarrow \textsf{out} \cup \trajxu$\;
 }
 \caption{Hit-and-run}
\end{algorithm}
\end{minipage}
\hspace{-8pt}\begin{minipage}{.43\textwidth}
\centering
\includegraphics[width=1\linewidth]{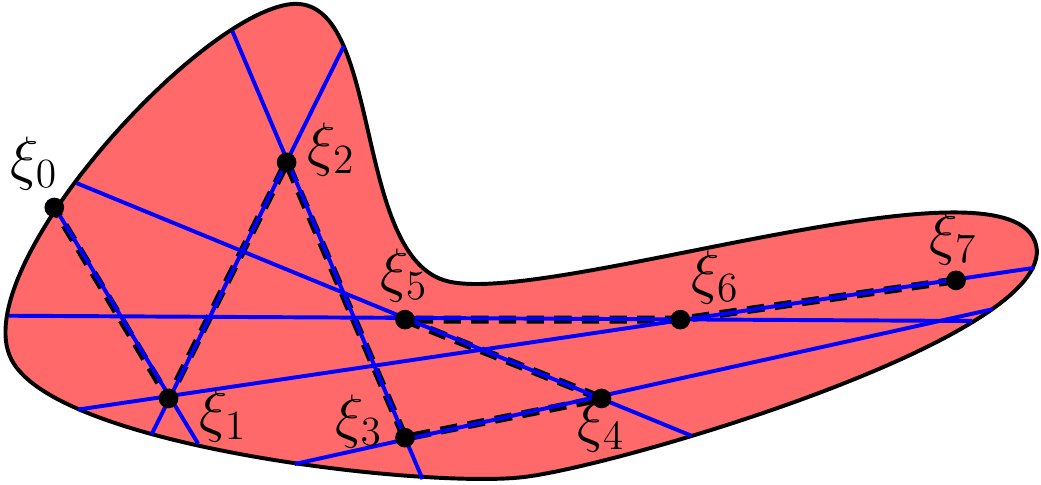}
\captionof{figure}{Illustration of hit-and-run. Blue lines denote sampled random directions, black dots denote samples.} \label{fig:hnr}
\end{minipage}%
\vspace{-10pt}
\subsubsection{Ellipsoid hit-and-run}\label{sec:elllipsoidsampling}
When we have a linear system with quadratic cost and convex control constraints - a very common setup in the optimal control literature - $\trajset_\unsafeset^{\trajxu^*} \doteq \{ \trajxu \ |\ c(\trajxu) < c(\trajxu^*)\} \cap \dynfeasibleset \equiv \{ \trajxu \ |\ \trajxu^\top V \trajxu < \trajxu^{*^\top} V \trajxu^*\} \cap \dynfeasibleset$ is an ellipsoid in the trajectory space, which can be efficiently sampled via a specially-tailored hit-and-run method. Here, the quadratic cost is written as $c(\trajxu)\doteq \trajxu^\top V \trajxu$, where $V$ is a matrix of cost parameters, and we omit the control and task constraints for now. Without dynamics, the endpoints of the line $\mathcal{L}$, $L_-, L_+$, (c.f. Alg. \ref{alg:hnr}), can be found by solving a quadratic equation $(\trajxu+\beta r)^\top V(\trajxu+\beta r) = \trajxu^{*^\top} V \trajxu^*$. We show that this can still be done with linear dynamics by writing $\trajset_\unsafeset^{\trajxu^*}$ in a special way. $\dynfeasibleset$ can be written as an eigenspace of a singular ``dynamics consistency" matrix, $\dynproject_1$, which converts any arbitrary state-control trajectory to one that satisfies the dynamics, one time-step at a time. Precisely, if the dynamics can be written as $\state_{t+1} = A\state_t + B\control_t$, we can write a matrix $\dynproject_1$:
\begin{equation}\label{eq:d1}
\vspace{-8pt}
\small
	\underbrace{\begin{bmatrix}
		\state_1 \\ \control_1 \\ \state_2 \\ \control_2 \\ \tilde{\state}_3 \\ \vdots \\ \control_{T-1} \\ \tilde{\state}_T
	\end{bmatrix}}_{\hat\xi_{xu}} = \underbrace{\begin{bmatrix}
		I & 0 & 0 & 0 & 0 & \cdots & \cdots & 0\\
		0 & I & 0 & 0 & 0 & \cdots & \cdots & 0\\
		A & B & 0 & 0 & 0 & \cdots & \cdots & 0\\
		0 & 0 & 0 & I & 0 & \cdots & \cdots & 0\\
		0 & 0 & A & B & 0 & \cdots & \cdots & 0\\
		\vdots & \vdots & \vdots & \vdots & \vdots & \ddots & \ddots & \vdots \\
		0 & 0 & 0 & \cdots & \cdots & 0 & I & 0 \\
		0 & 0 & 0 &\cdots & \cdots & A & B & 0 \\
	\end{bmatrix}}_{\dynproject_1} 
	\underbrace{\begin{bmatrix}
		\state_1 \\ \control_1 \\ \tilde{\state}_2 \\ \control_2 \\ \tilde{\state}_3 \\ \vdots \\ \control_{T-1} \\ \tilde{\state}_T
	\end{bmatrix}}_{\xi_{xu}}
\end{equation}
that fixes the controls and the initial state and performs a one-step rollout, replacing the second state with the dynamically correct state. In Eq. \ref{eq:d1}, we denote by $\tilde{x}_{t+1}$ a state that cannot be reached by applying control $\control_t$ to state $\state_t$. Multiplying the one-step corrected trajectory $\hat\xi_{xu}$ by $\dynproject_1$ again changes $\tilde{\state}_3$ to the dynamically reachable state $\state_3$. Applying $\dynproject_1$ to the original $T$-time-step infeasible trajectory $T-1$ times results in a dynamically feasible trajectory, $\xi_{xu}^\text{feas} = \dynproject_1^{T-1} \xi_{xu}$. Further, note that the set of dynamically feasible trajectories is $\dynfeasibleset \doteq \{ \xi_{xu} \ |\ \dynproject_1 \xi_{xu} = \xi_{xu} \}$, which is the span of the eigenvectors of $\dynproject_1$ associated with eigenvalue $1$. Thus, obtaining a feasible trajectory via repeated multiplication is akin to finding the eigenspace via power iteration \cite{power_iteration}. One can also interpret this as propagating through the dynamics with a fixed control sequence. Now, we can write $\trajset_\unsafeset^{\trajxu^*}$ as another ellipsoid which can be efficiently sampled by finding $L_-, L_+$ by solving a quadratic equation:
\vspace{-6pt}
\begin{equation}\label{eq:lin_intersection}
\trajset_\unsafeset^{\trajxu^*} \doteq \{\xi_{xu} \ |\ \xi_{xu}^\top \dynproject_1^{{T-1}^\top} V \dynproject_1^{T-1} \xi_{xu} \le \xi_{xu}^{*^\top} V \xi_{xu}^* \}.
\end{equation}
\vspace{-15pt}

We deal with control constraints separately, as the intersection of $\controlfeasibleset$ and Eq. \ref{eq:lin_intersection} is in general not an ellipsoid. To ensure control constraint satisfaction, we reject samples with controls outside of $\controlfeasibleset$; this works if $\controlfeasibleset$ is not measure zero. For task constraints, we ensure all sampled rollouts obey the goal constraints by adding a large penalty term to the cost function: $\tilde{c}(\cdot) \doteq c(\cdot) + \alpha_c\Vert \state_g - \state_T\Vert_2^2$, where $\alpha_c$ is a large scalar, which can be incorporated into Eq. \ref{eq:lin_intersection} by modifying $V$ and including $\state_g$ in $\trajxu$; all trajectories sampled in this modified set satisfy the goal constraints to an arbitrarily small tolerance $\varepsilon$, depending on the value of $\alpha_c$. The start constraint is satisfied trivially: all rollouts start at $\state_s$. Note the demonstration cost remains the same, since the demonstration satisfies the start and goal constraints; this modification is made purely to ensure these constraints hold for sampled trajectories. 
\vspace{-13pt}
\subsubsection{Convex hit-and-run}\label{sec:convexcost}

For general convex cost functions, the same sampling method holds, but $L_+, L_-$ cannot be found by solving a quadratic function. Instead, we solve $c(\xi_{xu} + \beta r) = c(\trajxu^*)$ via a root finding algorithm or line search.
\vspace{-25pt}
\subsubsection{Non-convex hit-and-run}\label{sec:nonconvexintersection}

If $\trajset_\unsafeset^{\trajxu^*}$ is non-convex, $\mathcal{L}$ can now in general be a union of disjoint line segments. In this scenario, we perform a ``backtracking" line search by setting $\beta$ to lie in some initial range: $\beta \in [\underbar{$\beta$}, \overline{\beta}]$; sampling $\beta_s$ within this range and then evaluating the cost function to see whether or not $\xi_{xu} + \beta_s r$ lies within the intersection. If it does, the sample is kept and hit-and-run proceeds normally; if not, then the range of possible $\beta$ values is restricted to $[\beta_s, \overline\beta]$ if $\beta_s$ is negative, and $[\underbar{$\beta$}, \beta_s]$ otherwise. Then, new $\beta$s are re-sampled until either the interval length shrinks below a threshold or a feasible sample is found. This altered hit-and-run technique still converges to a uniform distribution on the set in the limit, but has a slower mixing time than for the convex case, where mixing time describes the number of samples needed until the total variation distance to the steady state distribution is less than a small threshold \cite{nonconvexhnr}. Further, we accelerate sampling spread by relaxing the goal constraint to a larger tolerance $\hat\varepsilon > \varepsilon$ but keeping only the trajectories reaching within $\varepsilon$ of the goal.

%
\vspace{-12pt}
\subsection{Improving learnability using cost function structure}\label{sec:version_space}
\begin{wrapfigure}{L}{0.51\textwidth}
\begin{minipage}{0.52\textwidth}
\vspace{-9pt}
\begin{algorithm}[H]\label{alg:total}
\SetAlgoLined
\SetKwInOut{Input}{Input}
\SetKwInOut{Output}{Output}
\Output{$\occ \doteq \occ(\grid_1), \ldots, \occ(\grid_G)$}
\Input{$\traj_s = \{\traj_1^*, \ldots, \traj_{\numsafe}^*\}$, $c_\task(\cdot)$, $\textrm{known constraints}$\}}
 $\xi_{\textrm{u}} \leftarrow \{ \}$\;
 \For{i = 1:$\numsafe$}{
 \tcc{Sample unsafe $\traj$}
   \uIf{$\textrm{lin., quad., conv.}$}{
   $\xi_{\textrm{u}} \leftarrow \xi_{\textrm{u}} \cap \textsf{ellipsoidHNR}(\traj_i^*)$\;
   }\uElseIf{\textrm{lin., conv., conv.}}{
   $\xi_{\textrm{u}} \leftarrow \xi_{\textrm{u}}\cap\textsf{convexHNR}(\traj_i^*)$\;
   }\uElse{
   $\xi_{\textrm{u}} \leftarrow \xi_{\textrm{u}} \cap\textsf{nonconvexHNR}(\traj_i^*)$\;
   }
 }
  \tcc{Constraint recovery}
  \uIf{\textrm{prior, continuous}}{
  $\occ \leftarrow$ Problem \ref{prob:continuous_relaxation}$(\traj_s, \traj_u)$
  }\uElseIf{\textrm{prior, binary}}{
  $\occ \leftarrow$ Problem \ref{prob:integer_program}$(\traj_s, \traj_u)$
  }\Else{
  $\occ \leftarrow$ Problem \ref{prob:feasibility_program}$(\traj_s, \traj_u)$
  }
 \caption{Overall method}
\end{algorithm}
\end{minipage}
\vspace{-15pt}
\end{wrapfigure}
\vspace{-2pt}
Na\"ively, the sampled unsafe trajectories may provide little information. Consider an unsafe, length-$T$ discrete-time trajectory $\xi$, with start and end states in the safe set. This only says there exists at least one intermediate unsafe state in the trajectory, but says nothing directly about which state was unsafe. The weakness of this information can be made concrete using the notion of a version space. In machine learning, the version space is the set of consistent hypotheses given a set of examples \cite{aima}. In our setting, hypotheses are possible unsafe sets, and examples are the safe and unsafe trajectories. Knowing $\xi$ is unsafe only disallows unsafe sets that mark every element of the constraint space that $\xi$ traverses as safe: $(\occ(\grid_2) = 0) \wedge \ldots \wedge (\occ(\grid_{T-1}) = 0)$. If $\constraintspace$ is gridded into $G$ cells, this information invalidates at most $2^{G-T+2}$ out of $2^G$ possible unsafe sets. We could do exponentially better if we reduced the number of cells that $\xi$ implies could be unsafe.

We can achieve this by sampling sub-segments (or sub-trajectories) of the larger demonstrations, holding other portions of the demonstration fixed. For example, say we fix all but one of the points on $\xi$ when sampling unsafe lower-cost trajectories. Since only one state can be different from the known safe demonstration, the unsafeness of the trajectory can be uniquely localized to whatever new point was sampled: then, this trajectory will reduce the version space by at most a factor of $2$, invalidating at most $2^G - 2^{G-1} = 2^{G-1}$ unsafe sets.
One can sample these sub-trajectories in the full-length trajectory space by fixing appropriate waypoints during sampling: this ensures the full trajectory has lower cost and only perturbs desired waypoints. However, to speed up sampling, sub-trajectories can be sampled directly in the lower dimensional sub-trajectory space if the cost function $c(\cdot)$ that is being optimized is strictly monotone \cite{monotonicity}: for any costs $c_1, c_2 \in \mathbb{R}$, control $\control \in \controlset$, and state $\state \in \statespace$, $c_1 < c_2 \Rightarrow h(c_1, \state, \control) < h(c_2, \state, \control)$, for all $\state, \control$, where $h(c, \state, \control)$ represents the cost of starting with initial cost $c$ at state $\state$ and taking control $\control$. Strictly monotone cost functions include separable cost functions with additive or multiplicative stage costs, which are common in motion planning and optimal control. If the cost function is strictly monotone, we can sample lower-cost trajectories from sub-segments of the optimal path; otherwise it is possible that even if a new sub-segment with lower cost than the original sub-segment were sampled, the full trajectory containing the sub-segment could have a higher cost than the demonstration.

\vspace{-12pt}
\subsection{Integer program formulation}\label{sec:integer_program}
\vspace{-2pt}
After sampling, we can solve Problem \ref{prob:feasibility_program} to find an unsafe set consistent with the safe and unsafe trajectories. We now discuss the details of this process.
\textbf{Conservative estimate: } One can obtain a conservative estimate of the unsafe set $\unsafeset$ from Problem \ref{prob:feasibility_program} by intersecting all possible solutions: if the unsafeness of a cell is shared across all feasible solutions, that cell must be occupied. In practice, it may be difficult to directly find all solutions to the feasibility problem, as in the worst case, finding the set of all feasible solutions is equivalent to exhaustive search in the full gridded space \cite{mixed_integer_worst_case}. A more efficient method is to loop over all $G$ grid cells and set each one to be safe, and see if the optimizer can still find a feasible solution. Cells where there exists no feasible solution are guaranteed unsafe. This amounts to solving $G$ binary integer feasibility problems, which can be trivially parallelized. Furthermore, any cells that are known safe (from demonstrations) do not need to be checked. We use this method to compute the ``\emph{learned guaranteed unsafe set}", $\unsafesurelearned$, in Section \ref{sec:results}.

\noindent\textbf{A prior on the constraint: }As will be further discussed in Section \ref{sec:learnability}, it may be fundamentally impossible to recover a unique unsafe set. If we have some prior on the nature of the unsafe set, such as it being simply connected, or that certain regions of the constraint space are unlikely to be unsafe, we can make the constraint learning problem more well-posed. Assume that this prior knowledge can be encoded in some ``energy" function $E(\cdot, \ldots, \cdot) : \{0, 1 \}^G \rightarrow \reals$ mapping the set of binary occupancies to a scalar value, which indicates the desirability of a particular unsafe set configuration. Using $E$ as the objective function in Problem \ref{prob:feasibility_program} results in a binary integer program, which finds an unsafe set consistent with the safe and unsafe trajectories, and minimizes the energy:
\vspace{-2pt}
\begin{problem}[Inverse binary minimization constraint recovery]\label{prob:integer_program}
\vspace{-7pt}
\begin{equation}\label{eq:integer_program}
	\begin{array}{>{\displaystyle}c >{\displaystyle}l >{\displaystyle}l}
		\underset{\occ(\grid_1), \ldots, \occ(\grid_G) \in \{0, 1\}^G}{\text{minimize}} & \quad E(\occ(\grid_1), \ldots, \occ(\grid_G)) &\\
		\text{subject to} & \quad \sum_{\substack{\grid_i \in \{\phi(\traj^*_{s_j}(1)),\ldots,\\ \phi(\traj^*_{s_j}(T_j)) \} }} \occ(\grid_i) = 0, & \quad \forall j = 1, \ldots, \numsafe \\
		 & \quad \sum_{\substack{\grid_i \in \{\phi(\xi_{{\neg s}_k}(1)), \ldots,\\ \phi(\xi_{{\neg s}_k}(T_k))\}} } \occ(\grid_i) \ge 1, & \quad \forall k = 1, \ldots, \numunsafe
	\end{array}
\end{equation}
\end{problem}
\vspace{-5pt}

\noindent\textbf{Probabilistic setting and continuous relaxation: }A similar problem can be posed for a probabilistic setting, where grid cell occupancies represent beliefs over unsafeness: instead of the occupancy of a cell being an indicator variable, it is instead a random variable $\gridrv_i$, where $\gridrv_i$ takes value $1$ with probability $\occp(\gridrv_i)$ and value $0$ with probability $1-\occp(\gridrv_i)$. Here, the occupancy probability function maps cells to occupancy probabilities $\occp(\cdot) : \gridspace \rightarrow [0, 1]$.

Trajectories can now be unsafe with some probability. We obtain analogous constraints from the integer program in Section \ref{sec:integer_program} in the probabilistic setting. Known safe trajectories traverse cells that are unsafe with probability 0; we enforce this with the constraint $\sum_{\gridrv_i \in \phi(\xi_{s_j}^*)} \occp(\gridrv_i) = 0$: if the unsafeness probabilities are all zero along a trajectory, then the trajectory must be safe. Trajectories that are unsafe with probability $p_k$ satisfy $\sum_{\gridrv_i \in \phi(\xi_{{\neg s}_k})} \occp(\gridrv_i) = \expectation [ \sum_{\gridrv_i \in \phi(\xi_{{\neg s}_k})} \gridrv_i ] = (1-p_k)\cdot 0 + p_k \cdot S_k \ge p_k$ where we denote the number of unsafe grid cells $\phi(\xi_{{\neg s}_k})$ traverses when the trajectory is unsafe as $S_k$, where $S_k \ge 1$. The following problem directly optimizes over occupancy probabilities:
\vspace{-1pt}
\begin{problem}[Inverse continuous minimization constraint recovery]\label{prob:continuous_relaxation}
\vspace{-5pt}
\begin{equation}\label{eq:stochastic_program}
	\begin{array}{>{\displaystyle}c >{\displaystyle}l >{\displaystyle}l}
		\underset{\occ(\gridrv_1), \ldots, \occ(\gridrv_G) \in [0, 1]^G}{\text{minimize}} & \quad E(\occ(\gridrv_1), \ldots, \occ(\gridrv_G)) &\\
		\text{subject to} & \quad \sum_{\substack{\gridrv_i \in \{\phi(\traj^*_{s_j}(1)),\ldots,\\ \phi(\traj^*_{s_j}(T_j)) \} }} \occp(\gridrv_i) = 0, & \quad \forall j = 1, \ldots, \numsafe \\
		 & \quad \sum_{\substack{\gridrv_i \in \{\phi(\xi_{{\neg s}_k}(1)), \ldots,\\ \phi(\xi_{{\neg s}_k}(T_k))\}} } \occp(\gridrv_i) \ge p_k, & \quad \forall k = 1, \ldots, \numunsafe \\
	\end{array}
\end{equation}
\end{problem}
\vspace{-5pt}
When $p_k = 1$, for all $k$ (i.e. all unsafe trajectories are unsafe for sure), this probabilistic formulation coincides with the continuous relaxation of Problem \ref{prob:integer_program}. This justifies interpreting the solution of the continuous relaxation as occupancy probabilities for each cell.
Note that Problem \ref{prob:integer_program} and \ref{prob:continuous_relaxation} have no conservativeness guarantees and use prior assumptions to make the problem more well-posed. However, we observe that they improve constraint recovery in our experiments.

\vspace{-12pt}
\subsection{Bounded suboptimality of demonstrations}\label{sec:suboptimality}
\vspace{-5pt}
If we are given a $\delta$-suboptimal demonstration $\hat\xi$, where $c(\traj^*) \le c(\hat\xi) \le (1+\delta)c(\xi^*)$, where $\xi^*$ is an optimal demonstration, we can still apply the sampling techniques discussed in earlier sections, but we must ensure that sampled unsafe trajectories are truly unsafe: a sampled trajectory $\traj'$ of cost $c(\traj')\ge c(\traj^*)$ can be potentially safe. Two options follow: one is to only keep trajectories with cost less than $\frac{c(\hat\xi)}{1+\delta}$, but this can cause little to be learned if $\delta$ is large. Instead, if we assume a distribution on suboptimality, i.e. given a trajectory of cost $c(\hat\xi)$, we know that a trajectory of cost $c(\traj') \in [\frac{c(\hat{\traj})}{1+\delta}, c(\hat \traj)]$ is unsafe with probability $p_k$, we can then use these values of $p_k$ to solve Problem \ref{prob:continuous_relaxation}. We implement this in the experiments.

%% file: method_trunc.tex
\vspace{-12pt}
\section{Analysis}
\vspace{-5pt}
Due to space, the proofs/more remarks can be found in the appendix.
\vspace{-10pt}
\subsection{Learnability}\label{sec:learnability}
\vspace{-3pt}


%

We provide analysis on the learnability of unsafe sets, given the known constraints and cost function. Most analysis assumes unsafe sets defined over the state space: $\unsafeset \subseteq \statespace$, but we extend it to the feature space in Corollary \ref{thm:feature}. We provide some definitions and state a result bounding $\unsafesure$, the set of all states that can be learned guaranteed unsafe. We first define the signed distance:

\begin{definition}[Signed distance]
	Signed distance from point $p \in \mathbb{R}^m$ to set $\mathcal{S} \subseteq \mathbb{R}^m$, $\sd(p, \mathcal{S}) = -\inf_{y \in \partial\mathcal{S}} \Vert p - y \Vert$ if $p\in \mathcal{S}$; $\inf_{y \in \partial\mathcal{S}} \Vert p - y \Vert$ if $p\in \mathcal{S}^c$.
\end{definition}

%

\begin{theorem}[Learnability (discrete time)]\label{thm:umaxshell}
	For trajectories generated by a discrete time dynamical system satisfying $\Vert x_{t+1} - x_t \Vert \le \umax$ for all $t$, the set of learnable guaranteed unsafe states is a subset of the outermost $\umax$ shell of the unsafe set:
		$\unsafesure \subseteq \{ x \in \unsafeset \ |\ -\umax \le \sd(x, \unsafeset) \le 0 \}$ (see Section \ref{sec:app_learnability} for an illustration).
\end{theorem}
\vspace{-10pt}

%
%
\vspace{-5pt}
\begin{corollary}[Learnability (continuous time)]\label{thm:continuous_learnability}
	For continuous trajectories $\xi(\cdot): [0, T] \rightarrow \statespace$, the set of learnable guaranteed unsafe states shrinks to the boundary of the unsafe set: $\unsafesure \subseteq \{ x \in \unsafeset \ |\ \sd(x, \unsafeset) = 0 \}$.
\end{corollary}

Depending on the cost function, $\unsafesure$ can become arbitrarily small: some cost functions are not very informative for recovering a constraint. For example, the path length cost function used in many of the experiments (which was chosen due to its common use in the motion planning community), prevents any lower-cost sub-trajectories from being sampled from straight sub-trajectories. The system's controllability also impacts learnability: the more controllable the system, the more of the $\Delta x$ shell is reachable. We present a theorem quantifying when the dynamics allow unsafe trajectories to be sampled in Theorem \ref{thm:app_learnability_dynamics}.
\vspace{-5pt}
\vspace{-7pt}
\subsection{Conservativeness}
\vspace{-3pt}
We discuss conditions on $\unsafeset$ and discretization which ensure our method provides a conservative estimate of $\unsafeset$. For analysis, we assume $\unsafeset$ has a Lipschitz boundary \cite{dacorogna_calc_variations}. We begin with notation (explanatory illustrations are in Section \ref{sec:app_conservativeness}):

\begin{definition}[Set thickness]\label{def:thickness}
	Denote the outward-pointing normal vector at a point $p\in\partial\unsafeset$ as $\hat n(p)$. Furthermore, at non-differentiable points on $\partial \unsafeset$, $\hat{n}(p)$ is replaced by the set of normal vectors for the sub-gradient of the Lipschitz function describing $\partial\unsafeset$ at that point \cite{thickness}. The set $\unsafeset$ has a thickness larger than $d_\text{thick}$ if $\forall x \in \partial \unsafeset, \forall d \in [0, d_\textrm{thick}], \sd(x - d \hat n(x), \unsafeset) \le 0$.
\end{definition}

\begin{definition}[$\gamma$-offset padding]\label{def:offset}
	Define the $\gamma$-offset padding $\partial \unsafeset_{\gamma}$ as:
		$\partial \unsafeset_{\gamma} = \{ x \in \statespace \ | \ x = y + d \hat n(y), d\in [0, \gamma], y \in \partial \unsafeset \}$.
\end{definition}
\vspace{-15pt}
\begin{definition}[$\gamma$-padded set]\label{def:buffered_set}
	We define the $\gamma$-padded set of the unsafe set $\unsafeset$, $\unsafeset(\gamma)$, as the union of the $\gamma$-offset padding and $\unsafeset$: $\unsafeset(\gamma) \doteq \partial \unsafeset_\gamma \cup \unsafeset$.
\end{definition}
\vspace{-10pt}
\begin{corollary}[Conservative recovery of unsafe set]
	For a discrete-time system, a sufficient condition ensuring that the set of learned guaranteed unsafe states $\mathcal{A}_l^\textrm{rec}$ is contained in $\unsafeset$ is that $\unsafeset$ has a set thickness greater than or equal to $\umax$ (c.f. Definition \ref{thm:umaxshell}).
\end{corollary}
\vspace{-5pt}


If we use continuous trajectories directly, the guaranteed learnable set $\unsafesure$ shrinks to a subset of the boundary of the unsafe set, $\partial \unsafeset$ (c.f. Corollary \ref{thm:continuous_learnability}). However, if we discretize these trajectories, we can learn unsafe states lying in the interior, at the cost of conservativeness holding only for a padded unsafe set. For the following results, we make two assumptions, which are illustrated in Figs. \ref{fig:assumptions1} and \ref{fig:assumptions2} for clarity:

\noindent\textbf{Assumption 1}: The unsafe set $\unsafeset$ is aligned with the grid (i.e. there does not exist a grid cell $\grid$ containing both safe and unsafe states).

\noindent\textbf{Assumption 2}: The time discretization of the unsafe trajectory $\xi:[0,T]\rightarrow\statespace$, $\{t_1, \ldots, t_N\},t_i\in[0,T]$, for all $i$, is chosen such that there exists at least one discretization point for each cell that the continuous trajectory passes through (i.e. if $\exists t \in [0, T]$ such that $\xi(t) \in z$, then $\exists t_i \in \{t_1, \ldots, t_N\}$ such that $\traj(t_i)\in z$.
\vspace{-5pt}
\begin{theorem}[Continuous-to-discrete time conservativeness]\label{thm:c2d}
	Suppose that both Assumptions 1 and 2 hold. Then, the learned guaranteed unsafe set $\unsafesurelearned$, defined in Section \ref{sec:integer_program}, is contained within the true unsafe set $\unsafeset$.
	
	Now, suppose that only Assumption 1 holds. Furthermore, suppose that Problems 2-4 are using $M$ sub-trajectories sampled with Algorithm \ref{alg:hnr} as unsafe trajectories, and that each sub-trajectory is defined over the time interval $[a_i, b_i], i = 1,\ldots,M$. Denote $f_{\Delta x}([t_1,t_2]) \doteq \sup_{\state \in \statespace, \control \in \controlset, t\in[t_1,t_2]} \Vert f(\state, \control, t) \Vert\cdot(t_2-t_1)
$, and denote $[a^*, b^*]\doteq [a_j, b_j]$, where $j = \max_{i} f_{\Delta x}([a_i, b_i])$.
	Then, the learned guaranteed unsafe set $\unsafesurelearned$ is contained within the $f_{\Delta x}([a^*,b^*])$-padded unsafe set, $\unsafeset(f_{\Delta x}([a^*,b^*]))$.
\end{theorem}
\vspace{-10pt}
\begin{corollary}[Continuous-to-discrete feature space conservativeness]\label{thm:feature}
	Let the feature mapping $\phi(\state)$ from the state space to the constraint space be Lipschitz continuous with Lipschitz constant $L$. Then, under Assumptions 1 and 2 used in Theorem \ref{thm:c2d}, our method recovers a subset of the $Lf_{\Delta x}([a^*,b^*])$-padded unsafe set in the feature space, $\unsafeset(Lf_{\Delta x}([a^*,b^*]))$, where $[a^*,b^*]$ is as defined in Theorem \ref{thm:c2d}.
\end{corollary}

%% file: evaluations.tex
\vspace{-22pt}
\section{Evaluations}\label{sec:results}
\vspace{-8pt}
We provide an example showing the importance of using unsafe trajectories, and experiments showing that our method generalizes across system dynamics, that it works with discretization and suboptimal demonstrations, and that it learns a constraint in a feature space from a single demonstration.  See Appendix \ref{sec:app_experiments} for parameters, cost functions, the dynamics, control constraints, and timings.

\noindent\textbf{Version space example:} Consider a simple $5\times 5$ 8-connected grid world in which the tasks are to go from a start to a goal, minimizing Euclidean path length while staying out of the unsafe ``U-shape", the outline of which is drawn in black (Fig. \ref{fig:versionspace}). Four demonstrations are provided, shown in Fig. \ref{fig:versionspace} on the far left. Initially, the version space contains $2^{25}$ possible unsafe sets. Each safe trajectory of length $T$ reduces the version space at most by a factor of $2^{T}$, invalidating at most $2^{25}-2^{25-T}$ possible unsafe sets. Unsafe trajectories are computed by enumerating the set of trajectories going from the start to the goal at lower cost than the demonstration. The numbers of unsafe sets consistent with the safe and unsafe trajectories for varying numbers of safe trajectories are given in Table \ref{table:version_space}.

\begin{wrapfigure}{R}{0.37\linewidth}
\small
\begin{tabular}{ | p{1cm} || c |c| c| c| } 
\hline
 & 1 & 2 & 3 &  4 \\ 
\hline
Safe & 262144 & 4096 & 1024 & 256 \\ 
\hline
Safe \& unsafe & 11648 & 48 & 12 & 3 \\
\hline
\end{tabular}
\vspace{-7pt}
\captionof{table}{Number of consistent unsafe sets, varying the no. of demonstrations, using/not using unsafe trajectories. }
\label{table:version_space}
\end{wrapfigure}

Ultimately, it is impossible to distinguish between the three unsafe sets on the right in Fig. \ref{fig:versionspace}. This is because there exists no task where a trajectory with cost lower than the demonstration can be sampled which only goes through one of the two uncertain states. Further, though the uncertain states are in the $\Delta x$ shell of the constraint, due to the limitations of the cost function, we can only learn a subset of that shell (c.f. Theorem \ref{thm:umaxshell}).

There are two main takeaways from this experiment. First, by generating unsafe trajectories, we can reduce the uncertainty arising from the ill-posedness of constraint learning: after 4 demonstrations, using unsafe demonstrations enables us to reduce the number of possible constraints by nearly a factor of 100, from 256 to 3. Second, due to limitations in the cost function, it may be impossible to recover a unique unsafe set, but the version space can be reduced substantially by sampling unsafe trajectories.

\begin{figure}[t!]
  \centering
  \includegraphics[width=1\linewidth]{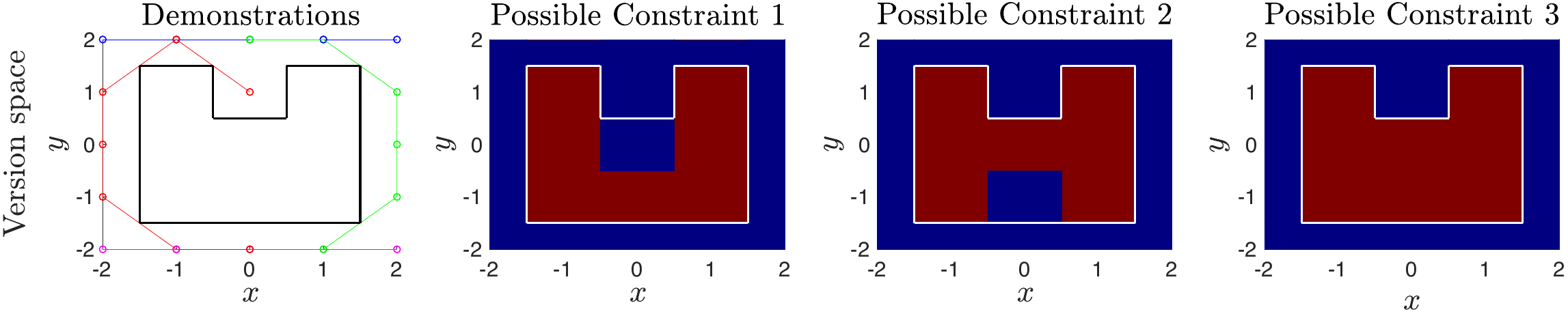}
  \vspace{-20pt}
  \caption{\textbf{Leftmost}: Demonstrations and unsafe set. \textbf{Rest}: Set of possible constraints. Postulated unsafe cells are plotted in red, safe states in blue.}
  \label{fig:versionspace}
  \vspace{-20pt}
\end{figure}%




\noindent\textbf{Dynamics and discretization:} Experiments in Fig. \ref{fig:experiment_tot} show that our method can be applied to several types of system dynamics, can learn non-convex/multiple unsafe sets, and can use continuous trajectories. The dynamics, control constraints, and cost functions for each experiment are given in Table \ref{table:experiment_dynamics} in Appendix \ref{sec:app_experiments}. All unsafe sets $\unsafeset$ are open sets. We solve Problems \ref{prob:integer_program} and \ref{prob:continuous_relaxation}, with an energy function promoting smoothness by penalizing squared deviations of the occupancy of a grid cell $\grid_i$ from its 4-connected neighbors $N(\grid_i)$: $\sum_{i=1}^G \sum_{\grid_j \in N(\grid_i)} \Vert \occ(\grid_i) - \occ(\grid_j) \Vert_2^2$. In all experiments, the mean squared error (MSE) is computed as $\frac{1}{G}\sqrt{\sum_{i=1}^G \Vert \occ(\grid_i)^* - \occ(\grid_i)\Vert_2^2}$, where $\occ(\grid_i)^*$ is the ground truth occupancy. The demonstrations are color-matched with their corresponding number on the $x$-axis of the MSE plots. For experiments with more demonstrations, only those causing a notable change in the MSE were color-coded. The learned guaranteed unsafe states $\unsafesurelearned$ are colored red on the left column.

We recover a non-convex ``U-shaped" unsafe set in the state space using trivial 2D single-integrator dynamics (row 1 of Fig. \ref{fig:experiment_tot}). The solutions to both Problems \ref{prob:continuous_relaxation} and \ref{prob:integer_program} return reasonable results, and the solution of Problem \ref{prob:integer_program} achieves zero error. The second row shows learning two polyhedral unsafe sets in the state space with 4D double integrator linear dynamics, yielding similar results. We note the linear interpolation of some demonstrations in row 1 and 2 enter $\unsafeset$; this is because both sets of dynamics are in discrete time and only the discrete waypoints must stay out of $\unsafeset$. The third row shows learning a polyhedral unsafe set in the state space, with time-discretized continuous, nonlinear Dubins' car dynamics, which has a 3D state $x \doteq \begin{bmatrix} \chi & y & \theta \end{bmatrix}^\top$. These dynamics are more constrained than the previous cases, so sampling lower cost trajectories becomes more difficult, but despite this we can still achieve near zero error solving Problem \ref{prob:integer_program}. Some over-approximation results from some sampled unsafe trajectories entering regions not covered by the safe trajectories. For example, the cluster of red blocks to the top left of $\unsafeset$ is generated by lower-cost trajectories that trade off the increased cost of entering the top left region by entering $\unsafeset$. This phenomenon is consistent with Theorem \ref{thm:app_c2d}; we recover a set that is contained within $\unsafeset(f_{\Delta x}[(0,T_\textrm{max}])$ (the maximum trajectory length $T_\textrm{max}$ was 14.1 seconds). Learning curve spikes occur when overapproximation occurs. Overall, we note $\unsafesurelearned$ tends to be a significant underapproximation of $\unsafeset$ due to the chosen cost function and limited demonstrations. For example, in row 1 of Fig. \ref{fig:experiment_tot}, $\unsafesurelearned$ cannot contain the portion of $\unsafeset$ near long straight edges, since there exists no shorter path going from any start to any goal with only one state within that region. For row 3 of Fig. \ref{fig:experiment_tot}, we learn less of the bottom part of $\unsafeset$ due to most demonstrations' start and goal locations making it harder to sample feasible control trajectories going through that region; with more demonstrations, this issue becomes less pronounced.



\begin{figure}[t!]
  \centering
  \includegraphics[width=\linewidth]{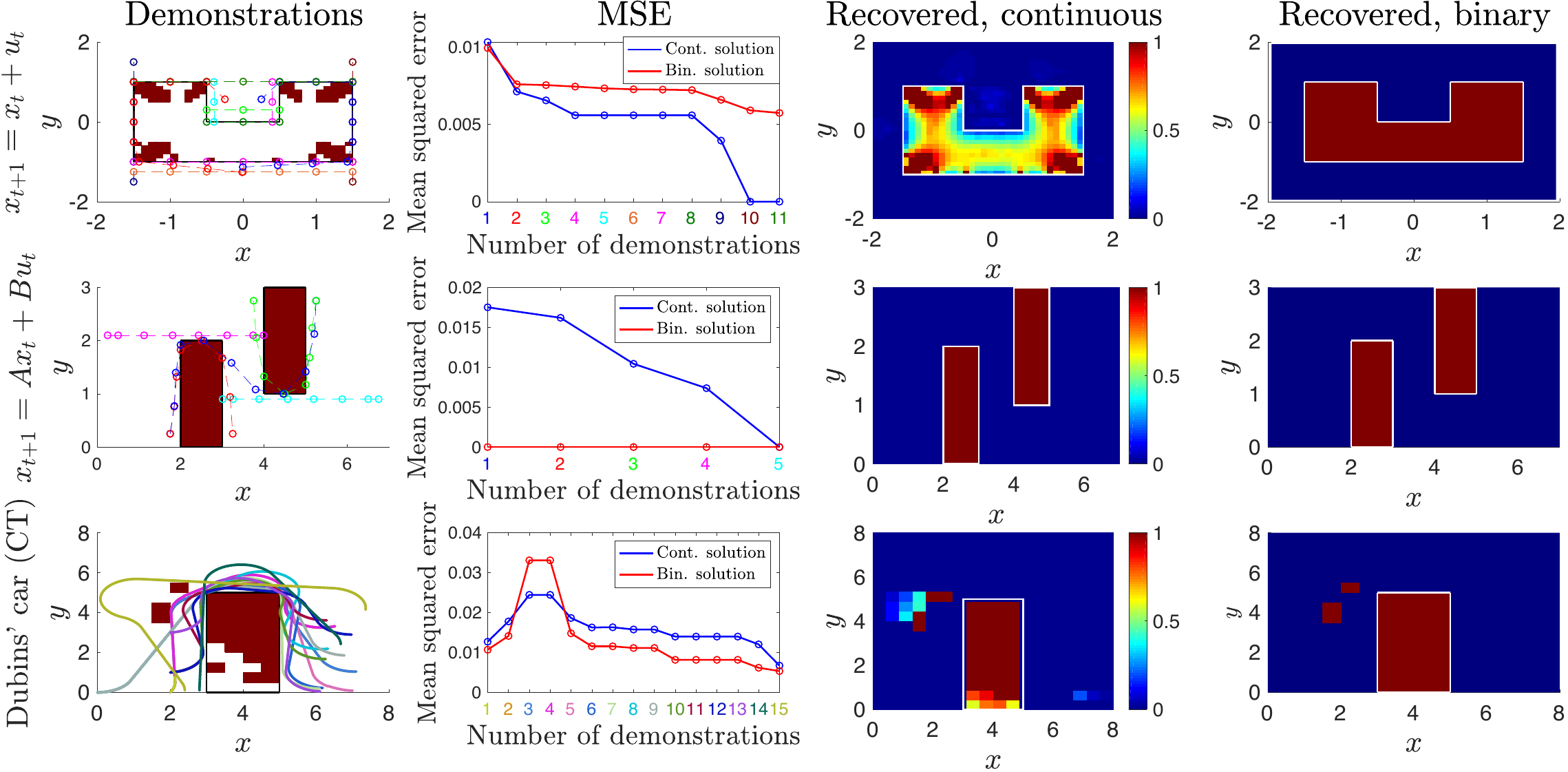}
  \vspace{-20pt}
  \caption{Results across dynamics, discretization. \textbf{Rows (top-to-bottom)}: Single integrator; double integrator; Dubins' car (CT). \textbf{Columns, left-to-right}: Demos., $\unsafeset$, $\unsafesurelearned$; MSE; Problem \ref{prob:continuous_relaxation} solution, all demos.; Problem \ref{prob:integer_program} solution, all demos.}
  \label{fig:experiment_tot}
  \vspace{-10pt}
\end{figure}%



\noindent\textbf{Suboptimal human demonstrations:} We demonstrate our method on suboptimal demonstrations collected via a driving simulator, using a car model with CT Dubins' car dynamics. Human steering commands were recorded as demonstrations, where the task was to navigate around the orange box and drive between the trees (Fig. \ref{fig:suboptimal}). For a demonstration of cost $c$, trajectories with cost less than $0.9c$ were believed unsafe with probability 1. Trajectories with cost $c'$ in the interval $[0.9c, c]$ were believed unsafe with probability $1-((c' - 0.9 c)/0.1c)$. MSE for Problem \ref{prob:continuous_relaxation} is shown in Fig. \ref{fig:suboptimal} (Problem \ref{prob:integer_program} is not solved since the probabilistic interpretation is needed). The maximum trajectory length $T_\textrm{max}$ is $19.1$ seconds; hence, despite suboptimality, the learned guaranteed unsafe set is a subset of $\unsafeset(f_{\Delta x}([0, T_\textrm{max}])$. While the MSE is highest here of all experiments, this is expected, as trajectories may be incorrectly labeled safe/unsafe with some probability.
\begin{figure}[t!]
  \centering
  \includegraphics[width=1\linewidth]{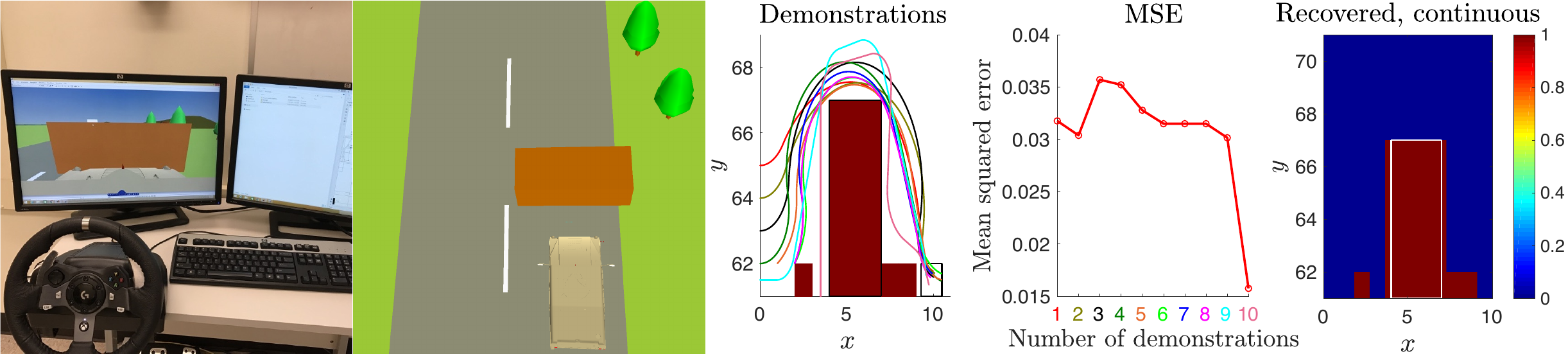}
  \vspace{-20pt}
  \caption{Suboptimal demonstrations: \textbf{left}: setup, \textbf{center}: demonstrations, $\unsafeset$, $\unsafesurelearned$, \textbf{center-right}: MSE, \textbf{right}: solution to Problem \ref{prob:continuous_relaxation}.}
  \label{fig:suboptimal}
  \vspace{-20pt}
\end{figure}

\noindent\textbf{Feature space constraint:} We demonstrate that our framework is not limited to the state space by learning a constraint in a feature space. Consider the scenario of planning a safe path for a mobile robot with continuous Dubins' car dynamics through hilly terrain, where the magnitude of the terrain's slope is given as a feature map (i.e. $\phi(x) = \Vert \partial H(\hat x)/\partial \hat x \Vert_2$, where $\hat x=[\chi\ y]^\top$ and $H(\hat x)$ is the elevation map). The robot will slip if the magnitude of the terrain slope is too large, so we generate a demonstration which obeys the ground truth constraint $\phi(\state) < 0.05$; hence, the ground truth unsafe set is $\unsafeset \doteq \{\state\ |\ \phi(\state) \ge 0.05\}$. From one safe trajectory (Fig. \ref{fig:transfer}) generated by RRT* \cite{rrt_star} and gridding the feature space as $\{0, 0.005, \ldots, 0.145, 0.15\}$, we recover the constraint $\phi(\state) < 0.05$ exactly.
\begin{wrapfigure}{r}{0.52\linewidth}
  \centering
    \vspace{-2pt}
  \includegraphics[width=\linewidth]{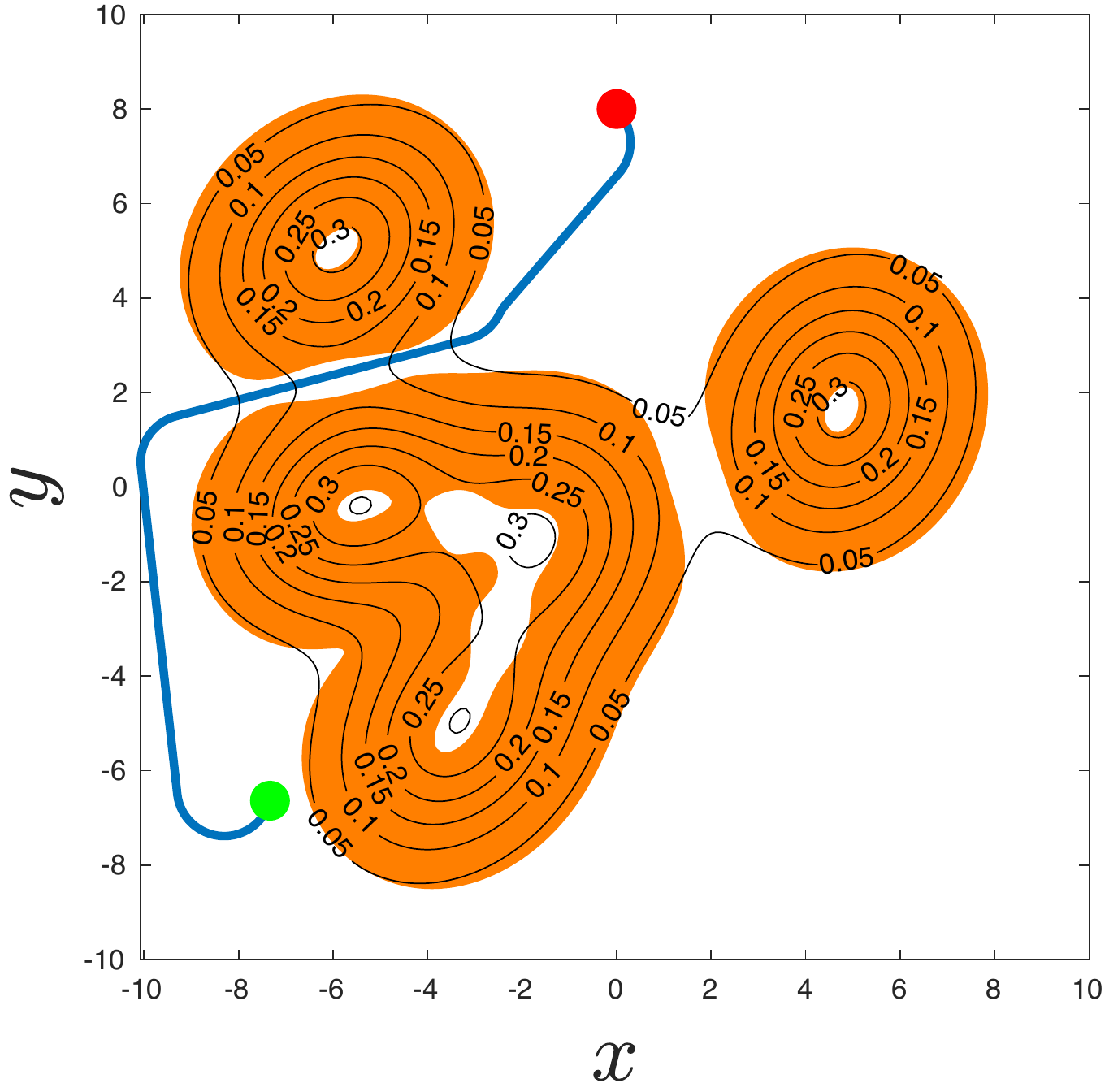}
  \vspace{-20pt}
  \caption{Demonstration (red: start, green: goal). Unsafe set $\unsafeset$ is plotted in orange. Terrain isocontours $H(x)=\textrm{const}$ are overlaid.}
  \label{fig:transfer}
  \vspace{-10pt}
\end{wrapfigure}
%

%% file: analysis.tex
\section{Analysis}

A brief overview of the most important results in this section:
\begin{itemize}
	\item Theorem \ref{thm:app_umaxshell} shows that all states that can be guaranteed unsafe must lie within some distance to the boundary of the unsafe set. Corollary \ref{thm:app_continuous_learnability} shows that the set of guaranteed unsafe states shrinks to a subset of the boundary of the unsafe set when using a continuous demonstration directly to learn the constraint.
	\item Corollary \ref{thm:app_conservative} shows that for the discrete time case and the continuous, non-discretized case, our estimate of the unsafe set is a guaranteed underapproximation of the true unsafe set if the unsafe set is sufficiently ``thick".
	\item For continuous trajectories that are then discretized, Theorem \ref{thm:app_c2d} shows us that the guaranteed unsafe set can be made to contain states on the interior of the unsafe set, but at the cost of potentially labeling states within some distance outside of the unsafe set as unsafe as well.
\end{itemize}

For convenience, we repeat the definitions and include some illustrations for the sake of visualization.
\vspace{-10pt}
\subsection{Learnability}\label{sec:app_learnability}


In this section, we will provide analysis on the learnability of unsafe sets, given the known constraints and cost function. Most of the analysis will be based off unsafe sets defined over the state space, i.e. $\unsafeset \subseteq \statespace$, but we will extend it to the feature space in Corollary \ref{thm:app_feature}. If a state $\state$ can be learned to be guaranteed unsafe, then we denote that $\state \in \unsafesure$, where $\unsafesure$ is the set of all states that can be learned guaranteed unsafe. 

We begin our analysis with some notation. 

\begin{definition}[Signed distance]
	Signed distance from point $p \in \mathbb{R}^m$ to set $\mathcal{S} \subseteq \mathbb{R}^m$, $\sd(p, \mathcal{S}) = -\inf_{y \in \partial\mathcal{S}} \Vert p - y \Vert$ if $p\in \mathcal{S}$; $\inf_{y \in \partial\mathcal{S}} \Vert p - y \Vert$ if $p\in \mathcal{S}^c$.
\end{definition}

The following theorem describes the nature of $\unsafesure$:


%

\begin{theorem}[Learnability (discrete time)]\label{thm:app_umaxshell}
	For trajectories generated by a discrete time dynamical system satisfying $\Vert x_{t+1} - x_t \Vert \le \umax$ for all $t$, the set of learnable guaranteed unsafe states is a subset of the outermost $\umax$ shell of the unsafe set:
		$\unsafesure \subseteq \{ x \in \unsafeset \ |\ -\umax \le \sd(x, \unsafeset) \le 0 \}$.
\end{theorem}
\begin{proof}
	Consider the case of a length $T$ unsafe trajectory $\xi = \{\state_1, \ldots, \state_N\}$, $\state_1 \in \unsafeset \vee \ldots \vee \state_T \in \unsafeset$. For a state to be learned guaranteed unsafe, $T-1$ states in $\xi$ must be learned safe. This implies that regardless of where that unsafe state is located in the trajectory, it must be reachable from some safe state within one time-step. This is because if multiple states in $\xi$ differ from the original safe trajectory $\xi^*$, to learn that one state is unsafe with certainty means that the others should be learned safe from some other demonstration. Say that $\state_1, \ldots, \state_{i-1}, \state_{i+1}, \ldots, \state_T \in \safeset$, i.e. they are learned safe. Since $(\Vert \state_{i+1} - \state_i \Vert \le \umax) \wedge (\Vert \state_{i} - \state_{i-1} \Vert \le \umax)$ and $\state_{i-1}, \state_{i+1} \in \safeset$, $\state_i$ must be within $\umax$ of the boundary of the unsafe set: $-\min_{y\in \partial \unsafeset} \Vert \state_i - y \Vert \ge \umax$, implying $-\umax \le \sd(x_i) \le 0$.
	
\end{proof}

\begin{figure}\label{fig:deltaxshell}
	\centering
	\includegraphics{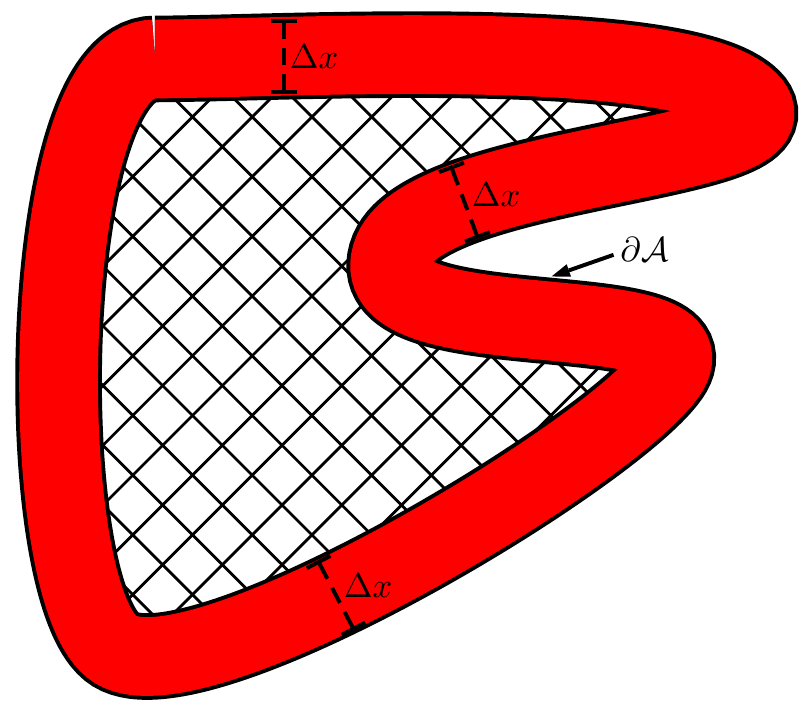}
	\caption{Illustration of the outermost $\Delta x$ shell (shown in red) of the unsafe set $\unsafeset$. The hatched area cannot be learned guaranteed safe.}
	\label{fig:padded}
\end{figure}
\begin{rem}
	For linear dynamics, $\Delta x$ can be found via
	\begin{equation}
		\underset{\state \in \statespace, \control \in \controlset}{\text{maximize}} \quad \Vert A \state + B\control - \state \Vert
	\end{equation}
	
	In the case of general dynamics, an upper bound on $\Delta x$ can be found via
	\begin{equation}
		\Delta x \le \sup_{\state \in \statespace, \control \in \controlset, t \in \{t_0, t_0+1, \ldots, T\}} \Vert f(\state, \control, t) - x\Vert
	\end{equation}
\end{rem}

\begin{corollary}[Learnability (continuous time)]\label{thm:app_continuous_learnability}
	For continuous trajectories $\xi(\cdot): [0, T] \rightarrow \statespace$, the set of learnable guaranteed unsafe states shrinks to the boundary of the unsafe set: $\unsafesure \subseteq \{ x \in \unsafeset \ |\ \sd(x, \unsafeset) = 0 \}$.
\end{corollary}
\begin{proof}
	The output trajectory of a continuous time system can be seen as the output of a discrete time system in the limit as the time-step is taken to 0. In this case, as long as the dynamics are locally Lipschitz continuous, $\Delta x \doteq \lim_{\Delta t \rightarrow 0} \Vert \state(t+\Delta t) - \state(t) \Vert \rightarrow 0$ \cite{khalil_nonlinear}, and via Theorem \ref{thm:app_umaxshell}, the corollary is proved.
\end{proof}

It is worth noting that depending on the cost function chosen, $\unsafesure$ can become arbitrarily small; in other words, some cost functions are more informative than others in recovering a constraint. An interesting avenue of future work is to investigate the properties of cost functions that enable more to be learned about the constraints and how this knowledge can help inform reward (or cost) shaping.

\subsubsection{Learnability (dynamics)}\label{sec:app_learnability_dynamics}

%
Depending on the dynamics of the system, it may be impossible to obtain sub-trajectories with few perturbed waypoints from sampling, due to there only being one feasible control sequence that takes the system from a start to a goal state. We formalize this intuition in the following theorem:

\begin{definition}[Forward reachable set]
	\sloppy The forward reachable set $\frs(\state_s, \controlset, T_1, T_2)$ is the set of all states that a dynamical system can reach at time $t = T_2$ starting from $\state_s$ at time $t = T_1$, using controls drawn from an admissible set of controls $\controlset$:
	\begin{equation}
		\frs(\state_s, \controlset, T_1, T_2) \doteq \{ z \in \statespace \ |\ \exists u(t):[T_1, T_2]\rightarrow\controlset, \state_{T_1} = \state_s, \state_{T_2} = z  \}
	\end{equation}
\end{definition}


\begin{theorem}[Learnability (dynamics)]\label{thm:app_learnability_dynamics}
	Let $\state_1^*, \ldots, \state_M^*$ be consecutive waypoints on a safe trajectory $\xi^*$ at times $t_1, \ldots, t_M$, with time discretization $\Delta t_i$ between states $\state_i^*$ and $\state_{i+1}^*$, where all but $\state_1^*, \state_M^*$ are free to move. Then, a necessary condition for being able to sample unsafe trajectories is that $\exists \state_2 \in \frs(\state_1^*, \controlset, t_1, t_1+\Delta t_1), \ldots, \exists \state_{M-1}\in \frs(\state_{M-2}, \controlset,t_{M-2}, t_{M-2}+\Delta t_{M-2}), \state_M^* \in \frs(\state_{M-1}, \controlset, t_{M-1}, t_{M-1}+\Delta t_{M-1})$ such that $\exists i \in \{2, \ldots, M-1\}: \state_i^* \ne \state_i$: i.e. there exists at least one state that the dynamics allow to be moved from the demonstrated trajectory.
\end{theorem}

\begin{proof}
	Proof by contradiction. Assume that there does not exist an $i\in\{2, \ldots, M-1\}$ such that $\state_i \ne \state_i^*$. Then, there exists no alternate sequence of controls taking the system from $\state_1^*$ to $\state_M^*$; hence no trajectories satisfying the start and goal constraints can be satisfied.
	
	Additionally, the same analysis can be used for continuous trajectories in the limit as the time-step between consecutive waypoints, $\Delta t$, goes to 0.
\end{proof}

\begin{rem}
	This implies that when the dynamics are highly restrictive, less of the unsafe set can be learned to be guaranteed unsafe, and the learnable subset of the $\Delta x$-shell of the unsafe set (as described in Theorem \ref{thm:app_umaxshell}) can become small.
\end{rem}

\subsection{Conservativeness}\label{sec:app_conservativeness}

For the analysis in this section, we will assume that the unsafe set has a Lipschitz boundary; informally, this means that $\partial \unsafeset$ can be locally described by the graph of a Lipschitz continuous function. A formal definition can be found in \cite{dacorogna_calc_variations}. We define some notation:

\begin{definition}[Set thickness]\label{def:app_thickness}
	Denote the outward-pointing normal vector at a point $p\in\partial\unsafeset$ as $\hat n(p)$. Furthermore, at non-differentiable points on $\partial \unsafeset$, $\hat{n}(p)$ is replaced by the set of normal vectors for the sub-gradient of the Lipschitz function describing $\partial\unsafeset$ at that point \cite{thickness}. The set $\unsafeset$ has a thickness larger than $d_\text{thick}$ if $\forall x \in \partial \unsafeset, \forall d \in [0, d_\textrm{thick}], \sd(x - d \hat n(x), \unsafeset) \le 0$.
\end{definition}

\begin{definition}[$\gamma$-offset padding]\label{def:app_offset}
	Define the $\gamma$-offset padding $\partial \unsafeset_{\gamma}$ as:
		$\partial \unsafeset_{\gamma} = \{ x \in \statespace \ | \ x = y + d \hat n(y), d\in [0, \gamma], y \in \partial \unsafeset \}$.
\end{definition}

\begin{definition}[$\gamma$-padded set]\label{def:app_buffered_set}
	We define the $\gamma$-padded set of the unsafe set $\unsafeset$, $\unsafeset(\gamma)$, as the union of the $\gamma$-offset padding and $\unsafeset$: $\unsafeset(\gamma) \doteq \partial \unsafeset_\gamma \cup \unsafeset$.
\end{definition}
\begin{figure}
	\centering
	\includegraphics{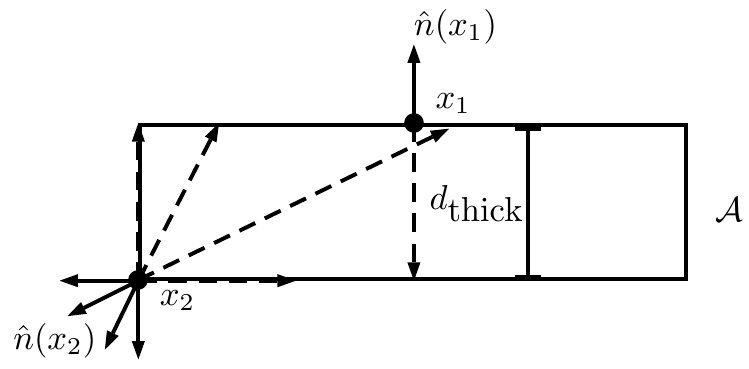}
	\caption{Illustration of thickness, c.f. Definition \ref{def:thickness}.}
	\label{fig:thickness}
\end{figure}

\begin{figure}
	\centering
	\includegraphics{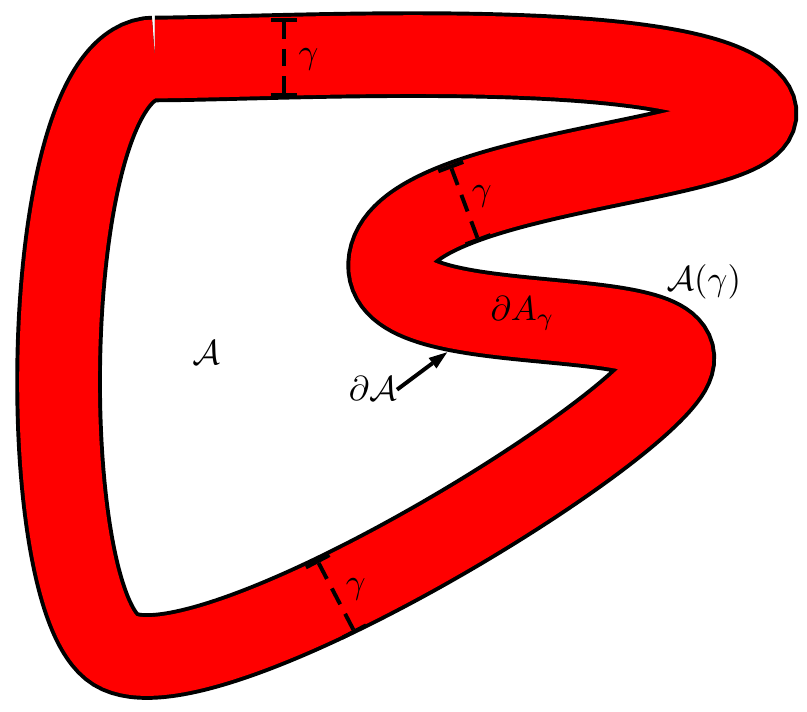}
	\caption{Illustration of the $\gamma$-padded set $\mathcal{A}(\gamma)$, which is the union of the red and white regions. The $\gamma$-offset padding is displayed in red. The original set $\mathcal{A}$ is shown in white.}
	\label{fig:padded}
\end{figure}

\begin{corollary}[Conservative recovery of unsafe set]\label{thm:app_conservative}
	For a discrete-time system, a sufficient condition ensuring that the set of recovered guaranteed unsafe states $\mathcal{A}_l^\textrm{rec}$ is contained in $\unsafeset$ is that $\unsafeset$ has a set thickness greater than or equal to $\umax$ (c.f. Definition \ref{thm:app_umaxshell}).
\end{corollary}
\begin{proof}
	Via Theorem \ref{thm:app_umaxshell}, our method will not determine that any state further inside than the outer $\umax$-shell is guaranteed unsafe. If $\unsafeset$ has thickness at least $\umax$, then our method will only determine states that are within the unsafe set to be guaranteed unsafe. This holds for discrete time dynamics and continuous time dynamics as well as $\umax\rightarrow 0$.
\end{proof}

Note that if we deal with continuous trajectories directly, the guaranteed learnable set shrinks to a subset of the boundary of the unsafe set, $\partial \unsafeset$. However, if we discretize these trajectories, we can learn unsafe states lying in the interior, at the cost of conservativeness guarantees holding only for a padded unsafe set.

The following results hold for continuous time trajectories. We begin the discussion with an intermediate result we will need for Theorem \ref{thm:app_c2d}:

\begin{lemma}[Maximum distance]\label{lem:app_maxdist}
	Consider a continuous time trajectory $\traj:[0,T]\rightarrow\statespace$. Suppose it is known that in some time interval $[a, b], a \le b, a, b \in [0, T]$, $\traj$ is unsafe; denote this sub-segment as $\traj([a, b])$. Further denote:
	\begin{equation}\label{eq:maxdist}
		f_{\Delta x}([t_1,t_2]) \doteq \sup_{\state \in \statespace, \control \in \controlset, t\in[t_1,t_2]} \Vert f(\state, \control, t) \Vert\cdot(t_2-t_1)
	\end{equation} 
	Consider any $t \in [a, b]$. Then, the signed distance from $\traj(t)$ to the unsafe set, $\sd(\traj(t), \unsafeset)$, is bounded by $\max(f_{\Delta x}([a,t]), f_{\Delta x}([t,b]))$.
\end{lemma}
\begin{proof}
	\begin{align*}
		\sup_{t\in[a,b]} \sd(\traj(t), \unsafeset) &= \max\Big(\sup_{\tau\in[a,t]}\sd(\xi(\tau), \unsafeset),\sup_{\tau\in[t,b]}\sd(\xi(\tau), \unsafeset)\Big)\\
		&\le \max\big(f_{\Delta x}([a,t]),f_{\Delta x}([t,b])\big)
	\end{align*}
\end{proof}

We introduce two assumptions, which are also illustrated in Figures \ref{fig:assumptions1} and \ref{fig:assumptions2} for clarity:

\noindent\textbf{Assumption 1}: The unsafe set $\unsafeset$ is aligned with the grid (i.e. there does not exist a grid cell $\grid$ containing both safe and unsafe states).

\noindent\textbf{Assumption 2}: The time discretization of the unsafe trajectory $\xi:[0,T]\rightarrow\statespace$, $\{t_1, \ldots, t_N\},t_i\in[0,T]$, for all $i$, is chosen such that there exists at least one discretization point for each cell that the continuous trajectory passes through (i.e. if $\exists t \in [0, T]$ such that $\xi(t) \in z$, then $\exists t_i \in \{t_1, \ldots, t_N\}$ such that $\traj(t_i)\in z$.
\begin{figure}
	\centering
	\includegraphics{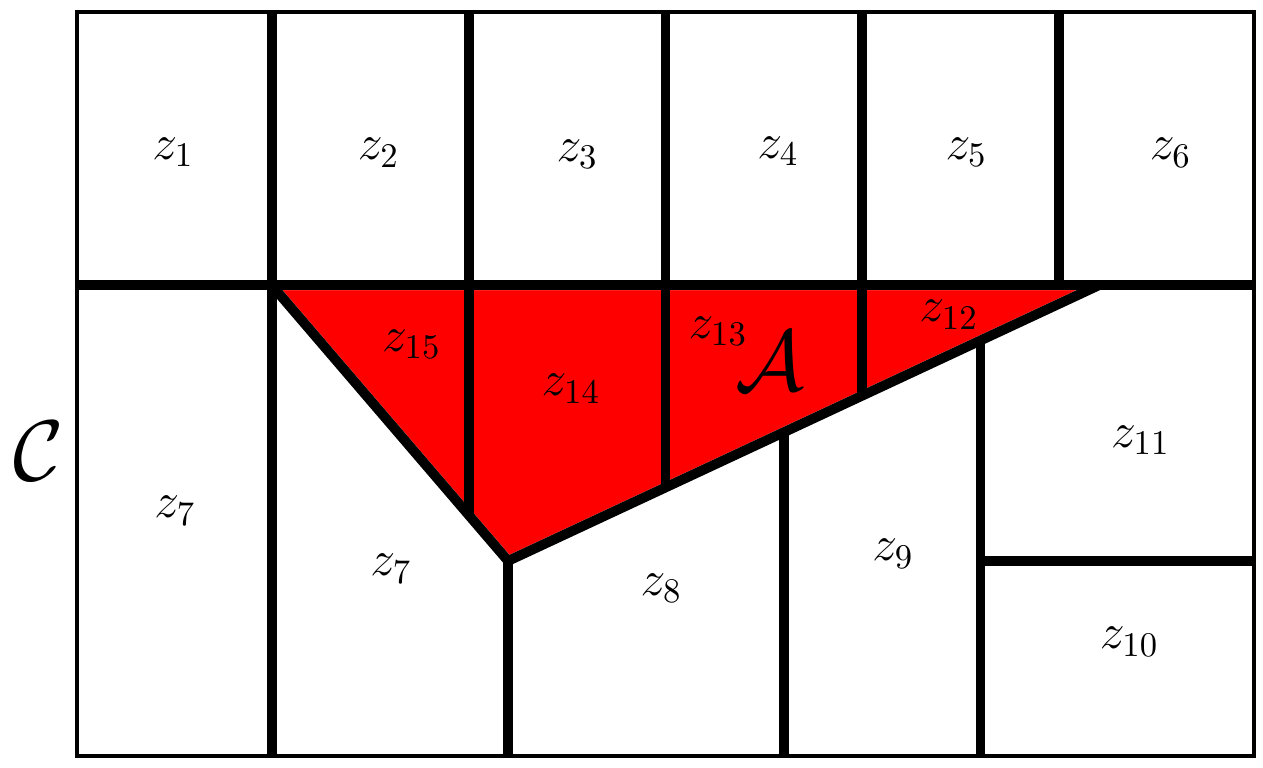}
	\caption{Illustration of Assumption 1 - all grid cells are either fully contained by $\unsafeset$ or $\unsafeset^c$.}
	\label{fig:assumptions1}
\end{figure}
\begin{figure}
	\centering
	\includegraphics{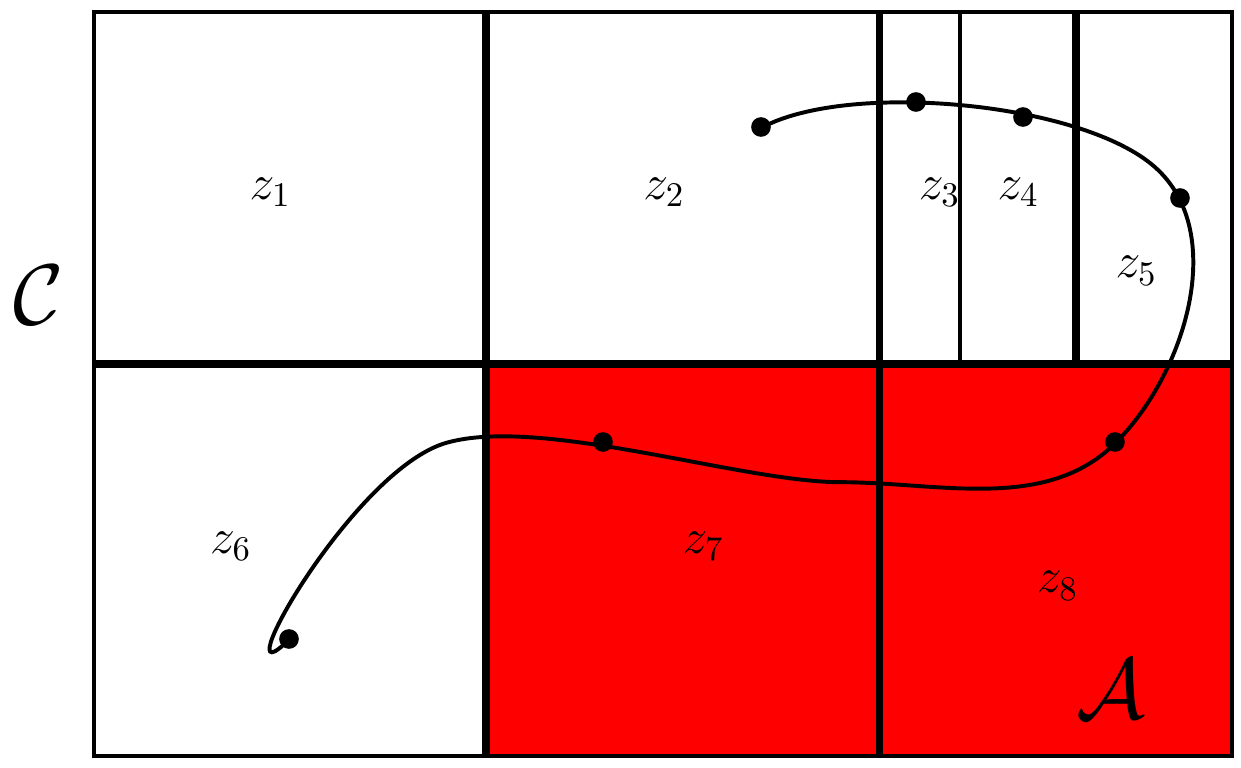}
	\caption{Illustration of Assumption 2: each cell $\grid$ that the trajectory passes through must have a time discretization point (shown as a dot).}
	\label{fig:assumptions2}
\end{figure}

We also introduce a convention for tie-breaking in Problems 2-4. Suppose there exists an unsafe trajectory $\traj$ for which a safe cell $z$ is incorrectly learned guaranteed unsafe. If a demonstration is added to the optimization problem which marks cell $z$ as safe, to avoid infeasibility, we remove the unsafe trajectory $\traj$ from the optimization problem.

\begin{theorem}[Continuous-to-discrete time conservativeness]\label{thm:app_c2d}
	Suppose that both Assumptions 1 and 2 hold. Then, the learned guaranteed unsafe set $\unsafesurelearned$, defined in Section \ref{sec:integer_program}, is contained within the true unsafe set $\unsafeset$.
	
	Now, suppose that only Assumption 1 holds. Furthermore, suppose that Problems 2-4 are using $M$ sub-trajectories sampled with Algorithm \ref{alg:hnr} as unsafe trajectories, and that each sub-trajectory is defined over the time interval $[a_i, b_i], i = 1,\ldots,M$. Denote $[a^*, b^*]\doteq [a_j, b_j]$, where $j = \max_{i} f_{\Delta x}([a_i, b_i])$.
	Then, the learned guaranteed unsafe set $\unsafesurelearned$ is contained within the $f_{\Delta x}([a^*,b^*])$-padded unsafe set, $\unsafeset(f_{\Delta x}([a^*,b^*]))$.
	
\end{theorem}
\begin{proof}
	Let's prove the case where both Assumptions 1 and 2 hold. By Assumption 1, all cells $\grid$ which contain unsafe states $\state \in \unsafeset$ must be fully contained in the unsafe set: $\grid\in\unsafeset$. Now, suppose there exists a trajectory $\traj:[0,T]\rightarrow \statespace$ which is unsafe (i.e it satisfies the known constraints and has lower cost than a demonstration). Then, there exists at least one $t\in[0,T]$ such that $\traj(t)\in\unsafeset$. By Assumption 2, there exists a discretization point $t_i\in[0,T]$ such that $\traj(t_i)$ lies within some cell $\grid$, and $\grid \in \unsafeset$ by Assumption 1. Hence, we will only learn grid cells within $\unsafeset$ to be unsafe: $\unsafesurelearned\subseteq\unsafeset$.
	
	Let's prove the case where only Assumption 1 holds. Suppose in this case, there exists a cell $\grid \notin \unsafeset$ which is truly safe, but for which we have no demonstration that says cell $\grid$ is safe. Now, suppose there exists an unsafe trajectory $\traj([a^*, b^*])$ passing through $\grid$ which violates Assumption 2. Suppose that $\traj(t_i)\in\grid$, and $\{t_1, \ldots, t_N\}$ is chosen such that for all $j \in \{1, \ldots, N\}\setminus \{i\}$, $\traj(t_i)$ belongs to a known safe cell. Then, we may incorrectly learn that $\grid\in\unsafesurelearned$, as we force at least one point in the sampled trajectory to be unsafe. Via Lemma \ref{lem:app_maxdist}, we know that $\traj(t_i)$ is at most $\max(f_{\Delta x}([a^*,t_i]), f_{\Delta x}([t_i,b^*]))$ signed distance away from $\unsafeset$. For this trajectory and choice of $t_i$, any learned guaranteed unsafe state must be contained in the $\max(f_{\Delta x}([a^*,t_i]), f_{\Delta x}([t_i,b^*]))$-padded unsafe set. For this to hold for all choices of $t_i$, we must pad the unsafe set by $\max_{t_i\in[a^*,b^*]}\big(\max(f_{\Delta x}([a^*,t_i]), f_{\Delta x}([t_i,b^*]))\big)$, which is bounded by $f_{\Delta x}([a^*,b^*])$.
\end{proof}
\begin{rem}
	In practice, we observe that the bound in Theorem \ref{thm:app_c2d} when using only Assumption 1 is quite conservative, and as more demonstrations are added to the optimization, using the tie-breaking rule described previously removes the overapproximations described by Theorem \ref{thm:app_c2d}. Furthermore, though the experiments are implemented using only Assumption 1, ensuring Assumption 2 also holds is straightforward as long as the grid cells are large enough such that finding a sufficiently fine time-discretization is efficient.
\end{rem}
\begin{corollary}[Continuous-to-discrete feature space conservativeness]\label{thm:app_feature}
	Let the feature mapping $\phi(\state)$ from the state space to the constraint space be Lipschitz continuous with Lipschitz constant $L$. Then, under Assumptions 1 and 2 used in Theorem \ref{thm:app_c2d}, our method recovers a subset of the $Lf_{\Delta x}([a^*,b^*])$-padded unsafe set in the feature space, $\unsafeset(Lf_{\Delta x}([a^*,b^*]))$, where $[a^*,b^*]$ is as defined in Theorem \ref{thm:app_c2d}.
\end{corollary}
\begin{proof}
	From the definition of Lipschitz continuity, $\Vert \phi(x) - \phi(y) \Vert \le L \Vert x - y \Vert$. From Theorem \ref{thm:app_c2d}, the unsafe set estimate is a subset of the $f_{\Delta x}([a^*,b^*])$-expanded estimate in the continuous space case. Using Lipschitz continuity, the value in the feature can at most change by $Lf_{\Delta x}([a^*,b^*])$ from the boundary of the true constraint set to the boundary of the padded set; hence, the statement holds.
\end{proof}

\vspace{-20pt}
\section{Experimental details}\label{sec:app_experiments}
\vspace{-10pt}
\begin{table}
\centering
\smaller
\begin{tabular}{ | c || p{3.5cm} | c| c| } 
\hline
\textbf{Figure} & \textbf{Dynamics} &  \textbf{Ctrl. constraints} & \textbf{Cost function} \\ 
\hline
Fig. \ref{fig:experiment_tot}, Row 1 & $x_{t+1} = x_t + u_t$ & $\Vert u_t \Vert \le 0.5$ & $\sum_{t=1}^{T-1} \Vert \control_t \Vert_2^2$ \\ 
\hline
Fig. \ref{fig:experiment_tot}, Row 2 & $x_{t+1} = Ax_t + Bu_t$, $A \doteq \exp\Bigg(\textrm{diag}\Bigg(\begin{bmatrix}0 & 1\\ 0 & 0\end{bmatrix}, \begin{bmatrix}0 & 1\\ 0 & 0\end{bmatrix}\Bigg)\Bigg)$, $B \doteq \displaystyle\int_0^1 \exp(A\tau) d\tau \begin{bmatrix}0 & 1 & 0 & 1\end{bmatrix}^\top$ & $|\control_t| \le \begin{bmatrix}20 & 10\end{bmatrix}^\top$ & $\sum_{i=1}^{T-1} \Vert \state_i - \state_{i+1} \Vert_2^2$ \\\hline
Fig. \ref{fig:experiment_tot}, Row 3 & $\dot x = \begin{bmatrix}
	\cos(\theta) & \sin(\theta) & \control\end{bmatrix}^\top$ & $|u| \le 1$ & $\sum_i \tau_{u_i}$\\\hline
Fig. \ref{fig:suboptimal} & $\dot x = \begin{bmatrix}
	\cos(\theta) & \sin(\theta) & \control\end{bmatrix}^\top$ & $|u| \le 1$ & $\sum_i \tau_{u_i}$\\\hline
Fig. \ref{fig:transfer} & $\dot x = \begin{bmatrix}
	\cos(\theta) & \sin(\theta) & \control\end{bmatrix}^\top$ & $|u| \le 1$ & $\sum_i \tau_{u_i}$\\\hline
\end{tabular}
\caption{Dynamics, control constraints, and cost functions used in experiments. }
\label{table:experiment_dynamics}
\end{table}
\begin{table}
\centering
\smaller
\vspace{-10pt}
\begin{tabular}{ | c || c|c| } 
\hline
\textbf{Figure} & \textbf{Timing (sampling trajectories)} & \textbf{Timing (constraint recovery)}\\ 
\hline
Fig. \ref{fig:experiment_tot}, Row 1 & 11.5 $\textrm{min}$ & 3 $\textrm{min}$\\ 
\hline
Fig. \ref{fig:experiment_tot}, Row 2 & 4.5 $\textrm{min}$ & 4.5 $\textrm{min}$\\\hline
Fig. \ref{fig:experiment_tot}, Row 3 & 2 $\textrm{hrs}$ &$4 \textrm{ min}$ \\\hline
Fig. \ref{fig:suboptimal} & 1 $\textrm{hr}$&$2 \textrm{ min}$ \\\hline
Fig. \ref{fig:transfer} & 30 $\textrm{min}$&$4 \textrm{ min}$ \\\hline
\end{tabular}
\caption{Approximate runtime. }
\label{table:experiment_dynamics}
\end{table}
\vspace{-10pt}
Here, $\sum_i \tau_{u_i}$ is the total time duration of applied control input (i.e. the time it took to go from start to goal). All experiments were conducted on a 4-core 2017 Macbook Pro with a 3.1 GHz Core i7. All code was implemented in MATLAB.

\begin{table}
\centering
\smaller
\begin{tabular}{ | c || c|c|c|c|c| } 
\hline
 & Fig. \ref{fig:experiment_tot}, Row 1 & Fig. \ref{fig:experiment_tot}, Row 2 & Fig. \ref{fig:experiment_tot}, Row 3 & Fig. \ref{fig:suboptimal} & Fig. \ref{fig:transfer} \\ 
\hline
Space discretization & 0.1 & 0.25 & 0.5 & 1 & 1\\ \hline
Number of trajectories & 300000 & 150000 & 10000 & 10000 & 10000\\\hline
$\varepsilon$ &n/a & n/a& $10^{-3}$&$10^{-3}$ &$10^{-3}$\\\hline
$\hat\varepsilon$ &n/a & n/a& $10^{-2}$&$10^{-2}$ &$10^{-2}$\\\hline
$\alpha_c$ & $10^{10}$& $10^4$& 1 &1 &1\\\hline
Minimum $\mathcal{L}$ length & n/a & n/a & $10^{-10}$ & $10^{-10}$ & $10^{-10}$\\\hline
$f_{\Delta x}^\textrm{max}$ & n/a & n/a & 4.53 & 6&6.96\\\hline
\end{tabular}
\caption{Parameters for each experiment. }
\label{table:experiment_dynamics}
\end{table}